\pdfoutput=1
\documentclass{article}

\usepackage[final]{neurips_2024}

\usepackage[utf8]{inputenc} 
\usepackage[T1]{fontenc}    
\usepackage{hyperref}       
\usepackage{url}            
\usepackage{booktabs}       
\usepackage{amsfonts}       
\usepackage{nicefrac}       
\usepackage{microtype}      
\usepackage{xcolor}         
\usepackage{amsmath}

\newtheorem{lemma}{Lemma}
\newtheorem{proof}{Proof}
\newtheorem{proposition}{Proposition}
\newtheorem{property}{Property}
\usepackage{booktabs}
\usepackage{adjustbox}
\usepackage{multirow}
\usepackage{graphicx}
\usepackage{array}
\usepackage{multirow}
\usepackage{colortbl}
\usepackage{amsmath}
\usepackage{graphicx}
\usepackage{multirow}
\usepackage{subcaption}
\usepackage{svg}
\usepackage{wrapfig}
\usepackage{amsmath}
\usepackage{float}

\title{Deep Graph Neural Networks via Posteriori-Sampling-based Node-Adaptive \\ Residual Module}

\author{%
Jingbo Zhou$^{1,2}$\thanks{Equal contribution.}, Yixuan Du$^{2,3}$\footnotemark[1], Ruqiong Zhang$^{2,3}$\footnotemark[1], Jun Xia$^{1}$, Zhizhi Yu$^{3}$, \\
\textbf{Zelin Zang$^{1}$, Di Jin$^{3}$, Carl Yang$^{4}$, Rui Zhang$^{2}$\thanks{Corresponding author.}, Stan Z. Li$^{1}$\footnotemark[2]} 
\\
$^{1}$Westlake University, $^{2}$Jilin University, $^{3}$Tianjin University, $^{4}$Emory University
\\
\texttt{\{zhoujingbo, stan.zq.li\}@westlake.edu.cn}
}

\begin{document}

\maketitle

\begin{abstract}
Graph Neural Networks (GNNs), a type of neural network that can learn from graph-structured data through neighborhood information aggregation, have shown superior performance in various downstream tasks. However, as the number of layers increases, node representations become indistinguishable, which is known as over-smoothing. To address this issue, many residual methods have emerged. In this paper, we focus on the over-smoothing issue and related residual methods. Firstly, we revisit over-smoothing from the perspective of overlapping neighborhood subgraphs, and based on this, we explain how residual methods can alleviate over-smoothing by integrating multiple orders neighborhood subgraphs to avoid the indistinguishability of the single high-order neighborhood subgraphs. Additionally, we reveal the drawbacks of previous residual methods, such as the lack of node adaptability and severe loss of high-order neighborhood subgraph information, and propose a \textbf{Posterior-Sampling-based, Node-Adaptive Residual module (PSNR)}. We theoretically demonstrate that PSNR can alleviate the drawbacks of previous residual methods. Furthermore, extensive experiments verify the superiority of the PSNR module in fully observed node classification and missing feature scenarios. Our code
is available at \href{https://github.com/jingbo02/PSNR-GNN}{https://github.com/jingbo02/PSNR-GNN}.
\end{abstract}

\section{Introduction}
GNNs have emerged in recent years as the most powerful model for processing graph-structured data and have demonstrated exceptional performance across various fields, such as social networks \cite{DBLP:conf/kdd/PerozziAS14}, recommender systems \cite{DBLP:conf/www/Fan0LHZTY19}, and drug discovery \cite{DBLP:conf/nips/DuvenaudMABHAA15}. Through the message-passing mechanism that propagates and aggregates information of neighboring nodes, GNNs provide a general framework for learning information on graph structure. Despite the remarkable success, according to previous studies \cite{DBLP:conf/aaai/LiHW18,DBLP:conf/icml/XuLTSKJ18}, GNNs show significant performance degradation as the number of layers increase.
One of the main reasons for this situation is over-smoothing \cite{DBLP:conf/aaai/LiHW18,DBLP:conf/iclr/OonoS20,DBLP:conf/icml/XuLTSKJ18,DBLP:conf/iclr/KlicperaBG19}. 
Over-smoothing refers to the phenomenon in which node representations become increasingly similar to each other as GNNs recursively aggregate more neighborhood information.
This indistinguishability will inevitably degrade the performance of deep GNNs, restricting their ability to effectively model long-range dependencies among multi-hop neighbors.

Several methods have recently been proposed to alleviate over-smoothing in deep GNNs. According to \cite{DBLP:journals/corr/abs-2303-10993}, these methods fall into three categories: Normalization and Regularization \cite{Zhou_2020,Zhao_2019}, Change of GNN dynamics \cite{DBLP:conf/nips/EliasofHT21}, and Residual connections \cite{DBLP:conf/icml/XuLTSKJ18,DBLP:conf/iccv/Li0TG19}.
Among all of them, the residual-based method is inspired by the success of residual neural networks (ResNets) \cite{DBLP:conf/cvpr/HeZRS16} in computer vision. This type of method introduces a residual connection to the GNNs architecture. 
For example, JKNet \cite{DBLP:conf/icml/XuLTSKJ18} learns node representations by aggregating the outputs of all previous layers at the last layer. DenseGCN \cite{DBLP:conf/iccv/Li0TG19} concatenates the results of the current layer and all previous layers as the node representations of this layer. 
APPNP \cite{DBLP:conf/iclr/KlicperaBG19} uses the initial residual connection to retain the initial feature information with probability $\alpha$, and utilizes information aggregated at the current layer with probability $1-\alpha$. GCNII \cite{DBLP:conf/icml/ChenWHDL20} shares a similar framework with APPNP, and it further introduces an identical mapping to avoid overfitting.

In this paper, we study the over-smoothing issue in GNNs, with a particular emphasis on residual methods. First, we revisit the over-smoothing phenomenon of GNNs from the new perspective of overlapping neighborhood subgraphs and explain the essential reason why the residual method can alleviate the over-smoothing.
In essence, these methods mainly use multiple neighborhood subgraph aggregations to alleviate the indistinguishability of the single neighborhood subgraph aggregation, thereby improving the performance of the model. On this basis, we find that these residual methods often lack node adaptivity in utilizing multi-order neighborhood subgraph information, and at the same time, 
they still struggle to mitigate information loss when dealing with high-order neighborhood subgraphs, which hinders further improvement in the performance of deep GNNs. Although some residual methods, such as DenseGNN, can avoid
these drawbacks, they tend to introduce more parameters at deeper layers. This can lead to significant memory consumption and is prone to gradient explosion, limiting the scalability of the methods.

Considering these limitations, we propose a \textbf{Posteriori-Sampling-based Node-Adaptative Residual Module (PSNR)}. More specifically, this module introduces a graph encoder to learn the posterior distribution of the required residual coefficients for each node in different layers with minor overhead. And then, we can obtain the specific fine-grained node-adaptive residual coefficients by sampling from the distribution.
The contributions of this paper are as follows:
\begin{itemize}
    \item[$\bullet$] \emph{Perspective:} 
    We revisit the over-smoothing issue from a novel perspective of high-order neighborhood subgraph coincidences and explain why the residual methods can alleviate it. Through this lens, we reveal several significant drawbacks of prior residual methods that limit the performance and scalability of GNNs.
    \item[$\bullet$] \emph{Method:} We propose PSNR, a lightweight and model-agnostic module to mitigate the drawbacks of previous residual methods and provide theoretical justification for its advantages. 
    \item[$\bullet$] \emph{Experiments:} Extensive experiments verify that the PSNR module can effectively mitigate oversmoothing and further improve the performance of GNNs, especially in the case of missing feature that require deep GNNs.
\end{itemize}

\section{Related Work}
\subsection{Notations}
A connected undirected graph is represented by $\mathcal{G}=(\mathcal{V}, \mathcal{E})$, where $\mathcal{V}=\left\{v_1, v_2, \ldots, v_N\right\}$ is the set of $N$ nodes and $\mathcal{E} \subseteq \mathcal{V} \times \mathcal{V}$ is the edge set. The node features are represented in the matrix $\mathbf{H} \in \mathbb{R}^{N \times d}$, where $d$ represents
 the length of the feature. Let $\mathbf{A} \in\{0,1\}^{N \times N}$ denotes the adjacency matrix and $\mathbf{A}_{i j}=1$ only if an edge exists between nodes $v_i$ and $v_j$. $\mathbf{D}\in \mathbb{R}^{N \times N}$ is the diagonal degree matrix, where each element $d_i$ represents the number of edges connected to node $v_i$. $\tilde{\mathbf{A}}=\mathbf{A}+\mathbf{I}$, $\tilde{\mathbf{D}}=\mathbf{D}+\mathbf{I}$ represent the adjacency matrix and degree matrix with self-loop, respectively. 

\subsection{Graph Neural Networks}
A GNN layer updates the representation of each node via aggregating itself and its neighbors’ representations. Specifically, a layer’s output $\mathbf{H}^{\prime}$ consists of new representations $\mathbf{h}^{\prime}$ of each node computed as: 
\begin{equation} 
\mathbf{h}_i^{\prime}=\operatorname{f}_\theta\left(\mathbf{h}_i, \operatorname{AGGREGATE}\left(\left\{\mathbf{h}_j \mid v_j \in \mathcal{V}, (v_i, v_j) \in \mathcal{E} \right\}\right)\right),
\end{equation}
where $\mathbf{h}_i^{\prime}$ indicates the new representation of node $v_i$ and $\operatorname{f}_\theta 
(\cdot)$ denotes the update function. The difference between different GNNs lies in the update function $\mathbf{f_\theta}(\cdot)$ and the $\operatorname{AGGREGATE}(\cdot)$ function, which are also key to the performance of GNNs. Graph Convolutional Network (GCN) \cite{DBLP:conf/iclr/KipfW17} is a classical massage-passing GNNs follows layer-wise propagation rule:
\begin{equation}
\mathbf{H}_{k+1} 
=
\sigma\left(\tilde{\mathbf{D}}^{-\frac{1}{2}} \tilde{\mathbf{A}} \tilde{\mathbf{D}}^{-\frac{1}{2}} \mathbf{H}_{k} \mathbf{W}_{k}\right),
\label{eq_GCN}
\end{equation}
where ${\mathbf{H}_{k}}$ is the feature matrix of the $k$-th layer, $\mathbf{W}_{k}$ is a layer-specific learnable weight matrix, $\sigma(\cdot)$ denotes an activation function.

\subsection{Residual connection}
Several works have used residual connection to alleviate the over-smoothing issue. 
Common residual connections for GNNs are summarized in Table~\ref{res-gnn}, where ${\mathbf{H}_{k}}$ represents the output of the $k$-th layer, $\mathbf{W}_{k}$ is a learnable weight matrix for the $k$-th layer, $\alpha$ serves as a hyperparameter denoting the residual coefficient, and $\sigma(\cdot)$ denotes an activation function. Additionally, in DenseGNN, $\operatorname{AGG}$ represents a function with the concatenation of outputs from all previous layers as the input to the current layer, while in JKNet, $\operatorname{AGG}$ refers to the aggregation of all previous representations through concatenation, max-pooling, or LSTM-attention only at the final layer. Details can
be found in Appendix \hyperref[app_b]{B}.

\begin{table}
  \caption{Common residual connections for GNNs.}
  \label{tab:Residual}
  \centering
  \resizebox{0.8\linewidth}{!}{
  \begin{tabular}{lll}
\toprule
    \textbf{Residual Connection} & \textbf{Corresponding GCN}& \textbf{Formula}   \\
    \midrule
    \text{Res} &\text{ResGCN~\cite{DBLP:conf/iccv/Li0TG19}}&$\mathbf{H}_k = \mathbf{H}_{k-1} + \sigma\left(\tilde{\mathbf{D}}^{-\frac{1}{2}} \tilde{\mathbf{A}} \tilde{\mathbf{D}}^{-\frac{1}{2}} \mathbf{H}_{k-1} \mathbf{W}_{k-1}\right)$\\
      \text{InitialRes} &\text{APPNP~\cite{DBLP:conf/iclr/KlicperaBG19}}& $\mathbf{H}_k = \left(1-\alpha\right)\tilde{\mathbf{D}}^{-\frac{1}{2}} \tilde{\mathbf{A}} \tilde{\mathbf{D}}^{-\frac{1}{2}} \mathbf{H}_{k-1} + \alpha\mathbf{H}$ \\
      \text{Dense} &\text{DenseGCN~\cite{DBLP:conf/iccv/Li0TG19}}&$\mathbf{H}_k=\operatorname{AGG}(\mathbf{H},\mathbf{H}_{1}, \ldots,\mathbf{H}_{k-1})$\\
      \text{JK} &\text{JKNet~\cite{DBLP:conf/icml/XuLTSKJ18}}&$\mathbf{H}_{output}=\operatorname{AGG}(\mathbf{H}_1, \ldots,\mathbf{H}_{k-1})$\\
  \bottomrule
  \end{tabular}
  }
  \label{res-gnn}
\end{table}

\section{Why does the residual method alleviate over-smoothing?}
\label{section3}
In this section, we revisit over-smoothing from the perspective of overlapping high-order neighborhood subgraphs. Based on this, we elucidate the role of various residual methods in alleviating over-smoothing and identify their shortcomings.

\subsection{Revisit over-smoothing from the perspective of neighborhood subgraphs overlapping}
\begin{figure}[h]
    \centering
\includegraphics[width=\linewidth, page=1]{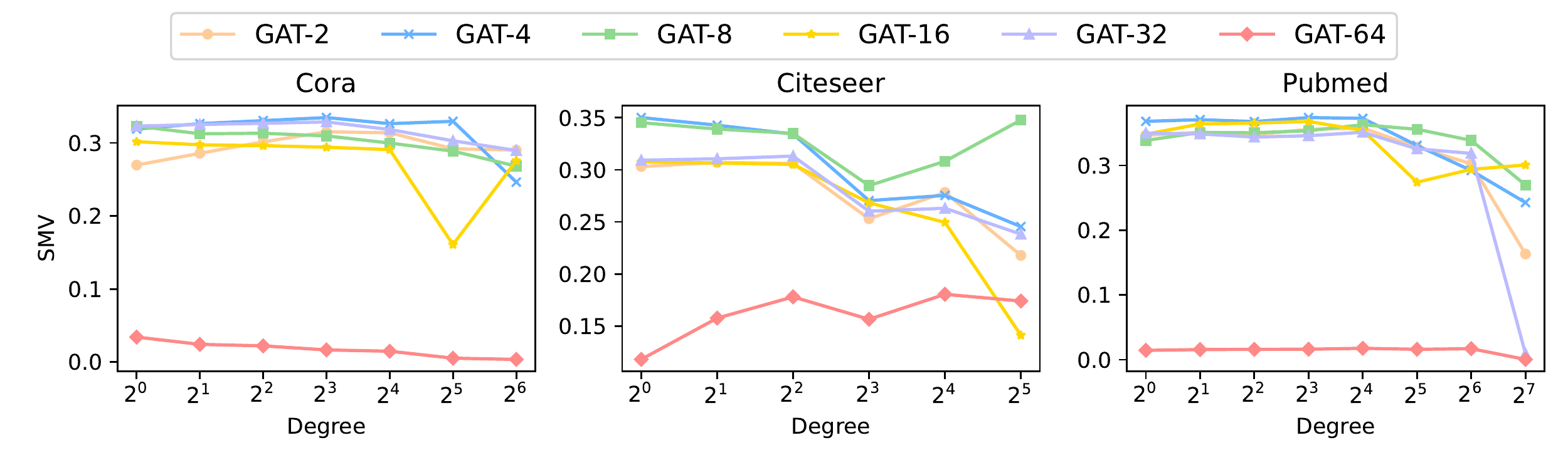}
    \caption{SMV for node groups of different degrees. More results are shown in Appendix \hyperref[app_c]{C}.}
    \label{fig:2_images}
    \vspace{-1em}
\end{figure}
For message-passing GNNs without the residual connection, the information domain of each node after $k$-layer aggregation is a corresponding $k$-order neighborhood subgraph.
Intuitively, the size of its $k$-order neighborhood subgraph grows exponentially as $k$
increases, leading to a significant increase in the overlap between the $k$-order neighborhood subgraphs of different nodes.
As a result, the aggregation result of different nodes on their respective $k$-order neighborhood subgraphs becomes indistinguishable.
This explanation can be partially validated from the perspective of node degrees. 
Considering nodes with high degrees tend to have larger neighborhood subgraph overlap, the correlation between neighborhood subgraph overlap and oversmoothiong can be validated if nodes with higher degree exhibit more pronounced over-smoothing.                              
To verify this point, we conduct experiments on three graph datasets: Cora, Citeseer, and Pubmed. Initially, nodes are grouped based on their degrees, with nodes having degrees falling within the range of $[2^i,2^{i+1})$ assigned to the $i$-th group. 
Subsequently, we perform aggregation with different layers of GCN and GAT and then calculate the degree of smoothing of the node representations within each group separately. The metric proposed in \cite{DBLP:conf/kdd/LiuGJ20} is used to measure the smoothness of the node representations within each group, namely smoothness metric value ($\operatorname{SMV}$), which calculates the average distances between the nodes within the group:
\begin{equation}
\operatorname{SMV}(\mathbf{X})=\frac{1}{N(N-1)} \sum_{i \neq j} \mathcal{D}\left(\mathbf{X}_{i,:}, \mathbf{X}_{j,:}\right),   
\end{equation}
where $\mathcal{D}(\cdot,\cdot)$ denotes the normalized Euclidean distance between two vectors:
\begin{equation}
\mathcal{D}(\mathbf{x}, \mathbf{y})=\frac{1}{2}\left\|\frac{\mathbf{x}}{\|\mathbf{x}\|}-\frac{\mathbf{y}}{\|\mathbf{y}\|}\right\|_2 . 
\end{equation}
From the definition, a smaller value of $\operatorname{SMV}$ indicates a greater similarity in node representations. We show the result of GAT in Figure~\ref{fig:2_images}. More results can be found in Appendix \hyperref[app_c]{C}.
It can be observed that the groups of nodes with higher degree tend to be more similar to each other within the group in different layers. This finding supports our claim.

\subsection{The role of residual method in alleviating over-smoothing}
After verifying the conclusion that neighborhood subgraph overlap leads to over-smoothing,
 a natural idea is to alleviate the overlap of the single neighborhood subgraph by utilizing multi-order neighborhood subgraph aggregations. In the following section, we show that the previous $k$-layer residual-based GNNs essentially represent different forms of utilizing neighborhood subgraph aggregations from 0 to $k$ orders.
\begin{table*}[h]
 \centering
   \caption{Utilization of neighborhood subgraphs by various residual methods.}
   \resizebox{0.6\linewidth}{!}{
	\begin{tabular}{c|c}
    \toprule
		\textbf{Model} & \textbf{Closed/Iterative form formulas}   \\
      \midrule 
		\text{ResGCN} &$\mathbf{H}_k = \sum^k_{j=0} \binom{j}{k}\mathbf{N}^j\mathbf{H}$ \\
		\text{APPNP}  &$
		\mathbf{H}_k = 
		\left( 1-\alpha\right)^k\mathbf{N}^k\mathbf{H} + \alpha
		\sum\limits_{j=0}^{k-1}\sum\limits_{i=0}^{j}
		\left( -1\right)^{j-i}
		\left(1-\alpha\right)^i\mathbf{N}^i\mathbf{H}
		$ \\
		\text{JKNet} &$\mathbf{H}_{k}=\operatorname{AGG}(\mathbf{N}\mathbf{H}, \ldots,\mathbf{N}^{k-1}\mathbf{H})$\\
		\text{DenseGCN} & 
		$ \mathbf{H}_k =
		\operatorname{AGG}(\mathbf{H}_{k-1},...,\mathbf{H}_1,\mathbf{H}_0)$\\
  \bottomrule
	\end{tabular}}
\label{tab:Residual2}
\end{table*}

In the rest of this paper, we take GCN, a classical residual-free message-passing GNN, as an example. Assuming that \textbf{H} is non-negative, the ELU function and the weight matrix can be ignored for simplicity. 
Combined with the formula of GCN in Eq.~\ref{eq_GCN}, the $k$-order neighborhood subgraph aggregation can be formulated as ${\mathbf{N}}^k\mathbf{H}$,
where $\mathbf{N} = \tilde{\mathbf{D}}^{-\frac{1}{2}}\tilde{\mathbf{A}}\tilde{\mathbf{D}}^{-\frac{1}{2}}$.
To show more intuitively how different residual models utilize
multi-order neighborhood subgraph aggregation $\mathbf{N}^k\mathbf{H}$, 
we rewrite their formula in Table \ref{tab:Residual2}.
Details of the derivation of the closed-form formula in this part are given in Appendix \hyperref[app_d]{D}. As can be observed, the output of GCN's residual-based variants contains multi-order matrix products that represent different order neighborhood subgraph aggregations from $0$ to $k$. There are two main ways to exploit them: \textbf{(1)} Summation, such as ResGCN and APPNP. Such methods employ linear summation over the aggregation of different order neighborhood subgraphs; \textbf{(2)} Aggregation functions, such as DenseGNN and JKNet. Such methods make direct and explicit exploitation of different order neighborhood subgraph aggregations through operations such as concatenation.

However, the utilization of multi-order neighborhood subgraph aggregations in these methods presents the following issues: Firstly, the summation methods all use a fixed coefficient to sum the neighborhood subgraph aggregations. Consequently, these methods inherently presume that the information from the neighborhood subgraph of the same order is equally important for different nodes, which lacks node adaptivity. 
Secondly, for ResGNN, APPNP, and JKNet, when the number of layers increases, the output of these methods still involves many high-order matrix products that are over-smoothed. This leads to severe information loss when aggregating high-order neighborhood subgraphs, which in turn degrades model performance at deeper layers. 
Although DenseGNN seems to alleviate the above issues to some extent, the recursive use of all previous neighborhood subgraph aggregation would introduce more parameters as the model deepens.
This increases memory consumption and raises the risk of gradient explosion at deeper layers.

\section{The Proposed Method PSNR}
\subsection{Methodology}
To solve the above issues, we propose a node-adaptive and lightweight residual module named PSNR. The motivation is to learn the adaptive residual coefficients for each node, thereby achieving fine-grained and node-level neighborhood subgraph aggregation to improve the performance of GNNs. However, directly learning these coefficients through backpropagation presents significant challenges. The primary challenge lies in the lack of transferability of learned coefficients. Specifically, in tasks such as semi-supervised node classification, nodes in the test and validation sets often cannot propagate information to the training nodes through multiple message-passing steps. Therefore, we cannot learn effective coefficients for these nodes during the training phase. 
\begin{wrapfigure}{r}{6cm}
	\centering
	\includegraphics[width=0.4\textwidth,page=1]{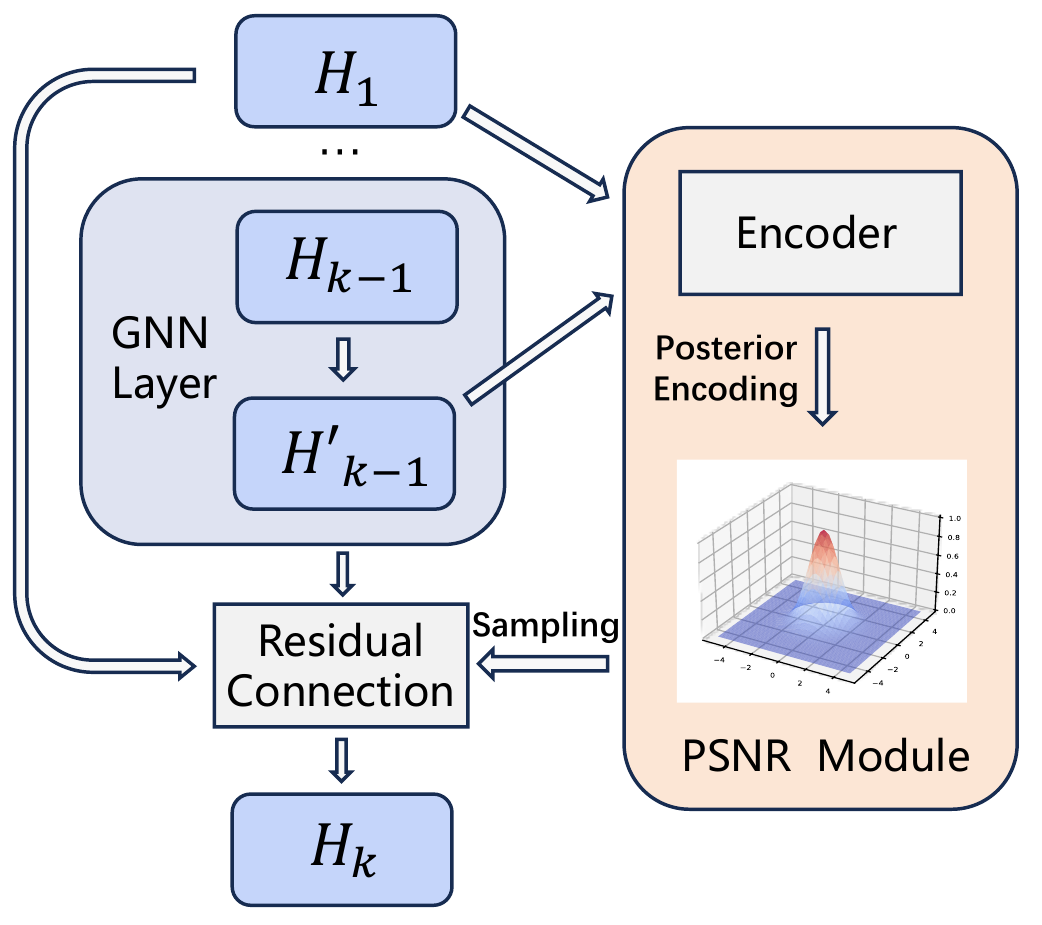}
	\caption{The framework of PSNR Module.}
	\label{fig:architecture}
\end{wrapfigure}
As a remedy, we regard node-level residual coefficients as hidden parameters. Our strategy involves estimating their posterior distribution $\mathbf{P}(\eta_k|\mathbf{A}, \mathbf{H}_k, k)$. In most cases, since the training, validation, and test sets originate from the same distribution, the posterior distribution learned from the training data possesses transferability to the validation and testing nodes. We can assume the posterior distribution to be Gaussian:
\begin{equation*}
	\mathbf{\eta}_{k}
	\sim \mathcal{N}
	(\mathbf{\mu}(\mathbf{A}, \mathbf{H}_k, k),{\mathbf{\sigma}^{2}(\mathbf{A}, \mathbf{H}_k, k)}).
\end{equation*}
A graph encoder can be used to model this distribution. This encoder leverages graph topology and node information as inputs. Subsequently, to enable back-propagation, we employ the reparameterization trick. This technique enables us to represent the sampling process from the aforementioned distribution as follows:
\begin{equation*}
	\mathbf{\eta}_{k} = \mathbf{\mu}(\mathbf{A}, \mathbf{H}_k, k) + \zeta \cdot \mathbf{\sigma}(\mathbf{A}, \mathbf{H}_k, k),  \;      
	\zeta \sim \mathcal{N}(0, 1).
\end{equation*}
Furthermore, we parameterize $\mathbf{\mu}(\mathbf{A}, \mathbf{H}_k, k)$ and $\mathbf{\sigma}(\mathbf{A}, \mathbf{H}_k, k)$ as an arbitrary GNN layer. 
While the posterior distributions of residual coefficients vary across different layers, we do not parameterize a specific encoder for each layer. Instead, we employ the same graph encoder and use the positional embedding generated from the layer number 
$k$ to differentiate the posterior distributions of residual coefficients for various layers. Consequently, the PSNR module can be formulated as follows:
\begin{align}
    \centering
	  \begin{aligned}
	&\;{\mathbf{H}^{\prime}_{k-1}} = 
	\operatorname{GraphConv}\left(\mathbf{H}_{k-1}\right) \\
	  & \;\mathbf{H}_k = \mathbf{H}_{1} + \phi(diag(\mathbf{\eta}_{k-1}))\left(\mathbf{H}_{1} - \mathbf{H}_{k-1}^{\prime}\right), \;	\mathbf{\eta}_{k-1}
	\sim \mathbf{\mathcal{N}}
	(\mathbf{\mu}_{k-1},{\mathbf{\sigma}_{k-1}}^{2})\\
	  & \;\mathbf{\mu}_{k}, \mathbf{\sigma}_{k} = \operatorname{GraphEncoder}_{\mu,  \sigma}(\mathbf{H}_{1} - \mathbf{H}_{k}^{\prime} + \gamma \operatorname{LayerEmb}(k)),\\
\end{aligned} 
\label{SNR_Defination}
\end{align}
where $\operatorname{GraphConv}(\cdot)$ is the $k$-th layer of any backbone GNN, $\mathbf{H}_k$ denotes the node representation matrix of the $k$-th layer, and ${\mathbf{H}^\prime_k}$ represents the output matrix of the $k$-th layer.
The first equation corresponds to the aggregation operation of the backbone GNN at the 
$k$-th layer.
$\eta_k$ represents the node-level residual coefficient at the $k$-th layer, where the element $\eta_k^{(i)}$ corresponds to the residual coefficient at the $i$-th node. In addition, $\eta_k$ follows a high-dimensional Gaussian distribution:
$\mathbf{\mathcal{N}}
(\mu_{k},{\sigma_{k}}^{2})$, and each time before each residual calculation, the distribution is first sampled to obtain the exact residual coefficients.
$\mu_{k}$ and ${\sigma_{k}}$ represent the mean and standard deviation of the distribution, respectively.
$diag(\eta_k)$ represents a diagonal matrix transformed from $\eta_k$, where the $i$-th diagonal element is precisely the $i$-th element of $\eta_k$.
$\phi(\cdot)$ represents the sigmoid function, which constrains the residual coefficient to $(0,1)$.
$\operatorname{GraphEncoder_{\mu, \sigma}}(\cdot)$ is a posterior encoder, which can be any graph convolution layer that is independent of the backbone GNN. $\operatorname{LayerEmb}(k)$ represents the positional embedding~\cite{vaswani2017attention} with layer number $k$ as input.
$\gamma$ is a learnable coefficient serving as the layer embedding factor. The analysis of the residual coefficients can be found in Appendix \hyperref[Appendix H]{H}.


Also, it is noteworthy that PSNR introduces randomness by sampling residual coefficients during both training and testing phases, thereby adding learnable random perturbations. This is different to other methods including DropEdge \cite{Rong_2019}, GRAND \cite{DBLP:conf/nips/FengZDHLXYK020} and DropMessage \cite{DBLP:conf/aaai/FangX0XY023} that only incorporate randomness during the training phase and primarily introduce perturbations to node features and graph structure. We provide the theoretical analysis for this design in Section \ref{sec:theory}. Additionally, we clarify the difference between PSNR and other subgraph-based methods in Appendix \hyperref[app_f]{F}.


\subsection{Theoretical justifications on
the advantages of the PSNR module}
\label{sec:theory}
In this section, we theoretically show that the PSNR module achieves finer-grained and node-adaptive neighborhood subgraph aggregations while avoiding the loss of high-order subgraph information.
Firstly, combining with Equation \ref{SNR_Defination}, the matrix form of the iterative  formula for PSNR-GCN is:
\begin{equation*}
	\mathbf{H}_{k}=
	\mathbf{H}_{1} + 
	\Lambda_{k-1}
	\left(
	\mathbf{H}_1 -
	\tilde{\mathbf{D}}^{-1 / 2} \tilde{\mathbf{A}} \tilde{\mathbf{D}}^{-1 / 2}\mathbf{H}_{k-1}
	\right).
\end{equation*}
For simplicity, we use $\Lambda_{k}$ to denote $diag(\eta_k)$.
Subsequently, based on this recursive form of formula, we derive the closed-form expression of PSNR-GCN as:
\begin{equation*}
\mathbf{H}_k  = 
\sum_{i=2}^{k-1} \prod_{j=i}^{k-1}\tilde{\mathbf{N}}_j
(
\mathbf{M}_i - 
\mathbf{M}_{i-1}
)
+ \prod_{i=1}^{k-1}\tilde{\mathbf{N}}_i\left( \mathbf{H}_1 + \mathbf{M}_1\right)
-\mathbf{M}_{k-1}, 
\end{equation*} 
where $\tilde{\mathbf{N}}_i = -\Lambda_{k-1} \mathbf{N}$, and 
$
\mathbf{M}_{k} = -\left( 
\Lambda_{k}
\mathbf{N} + \mathbf{I}\right)^{-1}\left(
\mathbf{I} + \Lambda_{k}
\right)\mathbf{H}_{1}
$.
The detailed derivation of this formula can be found in Appendix \hyperref[app_e]{E}. 
The first two terms consist of cumulative product terms of different orders, similar to the form of $\mathbf{N}^k\mathbf{H}$, thus approximating a new version of neighborhood subgraph aggregation. Additionally, the formula involves the aggregation of all neighborhood subgraphs from 1 to $k$ orders. This ensures that our method, like other residual methods, can comprehensively utilize multi-order neighborhood subgraph aggregations to enhance performance.
Furthermore, since $\Lambda_k$ is a diagonal matrix computed by a learnable posterior encoder with graph structure, node feature and layer number as input, the neighborhood subgraph aggregation of PSNR-GCN is fine-grained and node-adaptive. This sets it apart from methods like ResGCN and APPNP.

Beyond that, we can also demonstrate that the PSNR module can reduce the information loss of high-order subgraph aggregation, thereby further improving the performance of GNNs at deeper layers.
Specifically, we aim to prove that as the order $k$ increases, the smoothing rate of the cumulative product terms in PSNR is slower than that of the matrix power terms in other methods. Since we need to analyse the smoothing rate, which involves analyzing the asymptotic behavior of cumulative product terms or matrix power terms as the order increases, we can generalize the problem setup. Therefore, we only need to prove the following proposition to support our claim.
\begin{proposition}
\label{propos1}
Let $\mathcal{S} = \{\mathbf{S}_j = \Lambda_j \mathbf{N} | j \in \mathbb{N}_0 \}$, where $\Lambda_j$ represents a random diagonal matrix with each diagonal element $\Lambda_{j, ii}$ satisfying $0 < \Lambda_{j, ii} < 1$, and let $\mathbf{X}$ be any feature matrix. Then, as the order $k$ increases, the product of elements in the set $\mathcal{S}$ and the matrix $\mathbf{X}$, denoted as $\mathbf{X}^{(k)} = \prod_{i=1}^{k} \mathbf{S}_i \mathbf{X}$, converges to a rank-one matrix with identical rows slower than $\mathbf{X}_{GCN}^{(k)} = \mathbf{N}^k\mathbf{X}$.
\end{proposition}

Since the diagonal elements of the diagonal matrix $\Lambda$ are all results of the sigmoid function, we can assume that they have a lower bound $\epsilon > 0$. Therefore, in this setting, each element $\mathbf{S}_i$ in the set $\mathcal{S}$ is a row-stochastic matrix satisfying the following property.
\begin{property}
    \label{property}
    For a row-stochastic matrix $\mathbf{S}_k$, there exists an $\epsilon > 0$ satisfying the following conditions:\\
1. $\epsilon \leq \mathbf{S}_{k,ij} \leq 1$, if $(i,j) \in \mathcal{E}$,\\
2. $\mathbf{S}_{k,ij} = 0$, if $(i,j) \notin \mathcal{E}$.
\end{property}
We can refer to the conclusion in \cite{DBLP:conf/nips/WuAWJ23}. \cite{DBLP:conf/nips/WuAWJ23} describes the attention matrix of GAT as a row-stochastic matrix satisfying Property \ref{property}. Leveraging mathematical tools such as joint spectral radius from the perspective of dynamical systems, this work proves that the convergence rate of GAT is bounded below by GCN. Due to all the matrix $\mathbf{S}_i$ from $\mathcal{S}$ also satisfies Property \ref{property}, this conclusion can be borrowed to prove Proposition \ref{propos1}, thereby demonstrating that the PSNR module can alleviate the loss of high-order neighborhood subgraph information.
Furthermore, this conclusion also explains the introduction of randomness during both the training and testing phases. This random perturbation further reduces the smoothing rate, thereby enhancing the model's performance.

\subsection{Complexity analysis}
We take PSNR-GCN as an example to provide a complexity analysis of the PSNR module.
The time complexity of a vanilla GCN layer mainly comes from the matrix multiplication of \textbf{N} and \textbf{H}, hence its complexity is \emph{O}$(n^2d)$. And the main computational parts of PSNR module are the computation of mean and standard deviation, sampling of $p_k^{(i)}$, scalar multiplication and matrix addition, which correspond to a complexity of \emph{O}$(n^2d)$, \emph{O}($n$), \emph{O}($nd$), and \emph{O}($nd$), respectively. Thus the time complexity of the PSNR module is \emph{O}$(n^2d)$ and the time complexity of a GCN layer equipped the PSNR module is \emph{O}($n^2d$). As for space complexity, PSNR module needs to store the computed mean and variance for each node, i.e., \( O(2n) \), which can be approximately considered as \( O(n) \). Section \ref{ls} compares memory consumption with other residual methods on large graphs.





\section{Experiment}
In this section, we assess the performance of the PSNR module in comparison to other methods and answer the following research questions (\textbf{RQ}): \textbf{RQ1.} How well does the PSNR alleviate oversmoothing? \textbf{RQ2.} How does PSNR perform compared to other baseline when used
        with different backbone?
\textbf{RQ3.} Can  PSNR enable deeper networks to perform better under the missing feature scenario?
\textbf{RQ4.} How scalable is the PSNR on large graph datasets?
\subsection{Datasets}
We conducted experiments on ten real-world datasets, including three citation network datasets, i.e., Cora, Citeseer, Pubmed \cite{DBLP:journals/aim/SenNBGGE08}, two web network datasets, i.e., Chameleon and Squirrel  \cite{DBLP:journals/compnet/RozemberczkiAS21}, co-author/co-purchase network datasets, i.e., Coauthor-CS~\cite{shchur2018pitfalls}, Amazon-Photo~\cite{shchur2018pitfalls} and three larger datasets, i.e., 
Flickr~\cite{mcauley2012image}, Coauthor-Physics~\cite{shchur2018pitfalls} and Ogbn-arxiv \cite{DBLP:journals/corr/abs-2005-00687}. More
details of these datasets and data-splitting procedures can be found in Appendix \hyperref[app_g]{G}.

\subsection{Baselines and experimental settings} 
\label{sec:resourse}
We consider two fundamental GNNs, GCN \cite{DBLP:conf/iclr/KipfW17} and GAT \cite{Velickovic_2017}. For GCN, we test the performance of PSNR-GCN and its residual variant models, including ResGCN \cite{DBLP:conf/iccv/Li0TG19},  DenseGCN \cite{DBLP:conf/iccv/Li0TG19}, GCNII \cite{DBLP:conf/icml/ChenWHDL20} and JKNet \cite{DBLP:conf/icml/XuLTSKJ18}.
For GAT, we directly equip it with the following residual module: Res-GAT, InitialRes-GAT, Dense-GAT, JK-GAT and PSNR-GAT. And we adopt the GraphSAGE layer as the $\operatorname{GraphEncoder}$ of the PSNR module. The impact of different graph encoders on the experiments can be found in Appendix \hyperref[appi]{I}. In addition, we compare three recent representative methods belonging to different categories aimed at alleviating oversmoothing issues: DropMessage \cite{DBLP:conf/aaai/FangX0XY023} for the drop category, DeProp \cite{DBLP:conf/kdd/LiuHJLL23} for the norm category, and Half-Hop \cite{DBLP:conf/icml/AzabouGTLSLVVD23} for graph data processing. For the missing feature setting, we also conduct comparisons with several classical oversmoothing mitigation techniques, including BatchNorm \cite{DBLP:conf/icml/IoffeS15}, PairNorm \cite{Zhao_2019}, DGN \cite{Zhou_2020}, Decorr \cite{DBLP:conf/kdd/JinL0AT22}, DropEdge \cite{Rong_2019}.
For all baselines, the linear layers in the models are initialized with a standard normal distribution, and the convolutional layers are initialized with Xavier initialization. The Adam optimizer \cite{DBLP:journals/corr/KingmaB14} is used for training.  Experimental results are obtained from the server with four
core Intel(R) Xeon(R) Platinum 8358 CPUs @ 2.60GHZ,
one NVIDIA A100 GPU (80G), and models and datasets
used in this paper are implemented using the Deep Graph
Library (DGL) and Pytorch Geometric (PyG). Further details on the specific parameter settings can be found in Appendix \hyperref[app_g]{G}.

\subsection{Effectiveness in mitigating over-smoothing (RQ1)}
\label{exp:oversmoothing}
In this section, we aim to assess whether the PSNR module can mitigate the oversmoothing phenomenon in deep GNNs. We select representative datasets Cora, Amazon-Photo, and Chameleon. Using GCN as the backbone network, we compare our method against several residual methods: ResGCN, GCNII, JKNet, and DenseGNN. We set the number of layers to 2, 4, 8, 16, 32, 64 and test on ten random splits, with the average accuracy serving as the final result. The experimental results are depicted in Figure \ref{exp:fig1}.
\begin{figure}[htbp]
\centering\includegraphics[width=\textwidth]{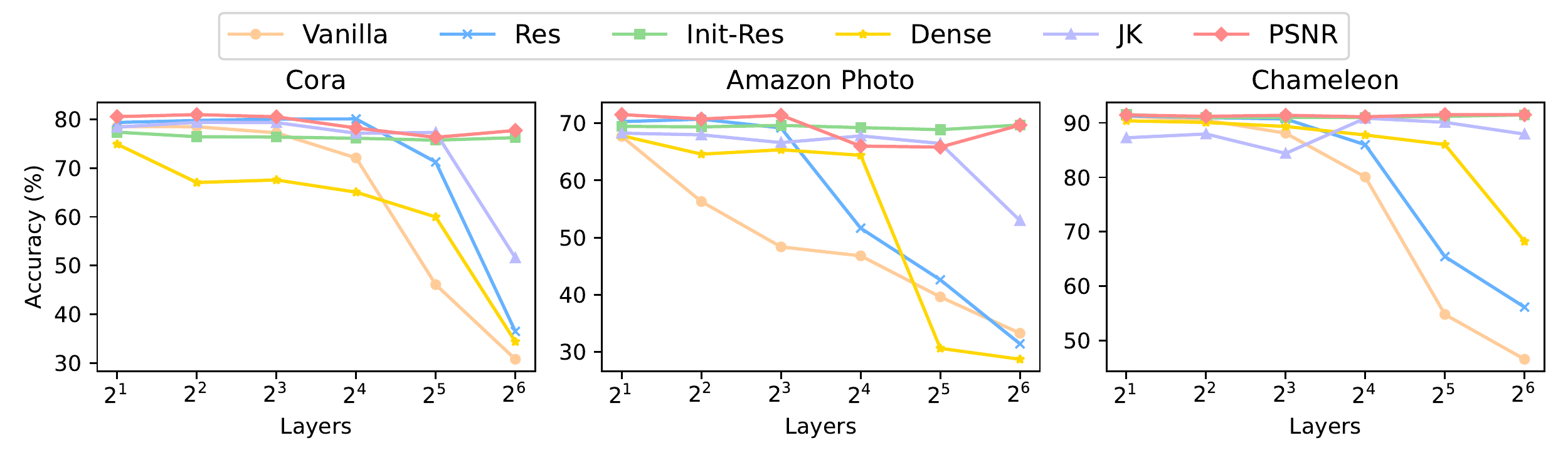}
	\caption{Different residual methods' effectiveness in mitigating over-smoothing.}
	\label{exp:fig1}
 \vspace{-0.3cm}
\end{figure}
Consistent with the analysis in the main text, most methods can alleviate over-smoothing, but at deeper layers, such as 64 layer, over-smoothing still occurs. In contrast, compared to other residual methods, PSNR maintains stable performance even at deeper layers, demonstrating remarkable effectiveness in mitigating over-smoothing in deep GNNs. This is attributed to PSNR module can effectively reduce the loss of high-order neighborhood subgraph information. It is noteworthy that among the various residual methods, another initial residual method GCNII also alleviate over-smoothing well. However, the subsequent experiment  will reveal that although the performance of GCNII remains relatively stable with varying layers, it leads to a decrease in overall performance.

\subsection{Fully observed feature setting (RQ2)}
\label{exp:full}
The PSNR module can effectively address the performance degradation of GNNs at deeper layers, but can it further enhance the overall performance of the model? In this section, we comprehensively evaluate the PSNR module across a wider range of datasets. Under the fully observed feature setting, we set the number of layers to 2, 4, 8, 16, 32, 64, and conduct experiments on ten random splits for each dataset, taking the average accuracy as the final result for each layer. The results of all layers of the model can be found in Appendix \hyperref[app_a]{A}. To evaluate the overall performance, we record the best results of each model in all layers for each dataset in Table \ref{ssnc}. 
\begin{table}
  \caption{Summary of classification accuracy (\%) results on real-world datasets. The best results are in bold, and the second best results are underlined.}
  \label{ssnc}
  \centering
  \resizebox{0.9\linewidth}{!}{
  \begin{tabular}{ccccccc|c}
\hline
\hline
\textbf{Method} & \textbf{Cora} & \textbf{Citeseer} & \textbf{CS} & \textbf{Photo} & \textbf{Chameleon} & \textbf{Squirrel} & \textit{Avg. Rank}\\
\midrule
GCN & 78.56$\pm$1.57 & 66.00$\pm$2.37 & 90.19$\pm$0.83 & 91.22$\pm$0.73 & 67.70$\pm$0.17 & 47.43$\pm$1.31 &6.67\\
ResGCN & 80.11$\pm$0.98 & 66.40$\pm$2.29 & 90.86$\pm$0.99 & 91.31$\pm$0.85 & \underline{70.70$\pm$1.54} & \underline{52.43$\pm$1.72} &\underline{3.67}\\
GCNII & 77.40$\pm$1.43 & 66.69$\pm$1.95 & 90.34$\pm$0.78 & \underline{91.46$\pm$0.84} & 69.69$\pm$1.36 & 46.13$\pm$1.59 &5.00\\
DenseGCN & 74.94$\pm$2.46 & 62.54$\pm$2.56 & 90.28$\pm$1.06 & 90.38$\pm$1.02 & 67.82$\pm$1.27 & 49.57$\pm$1.58 &7.17\\
JKNet & 79.42$\pm$1.48 & 65.49$\pm$2.32 & 90.94$\pm$0.89 & 90.85$\pm$1.19 & 68.26$\pm$1.15 & 49.87$\pm$1.44 &5.17\\
Half-Hop-GCN & \underline{80.74$\pm$1.80} & \underline{67.76$\pm$2.19} & 90.30$\pm$1.83 & 89.32$\pm$1.17 & 60.07$\pm$2.34 & 42.78$\pm$1.01 &6.33\\
DropMessage-GCN & 80.00$\pm$2.22 & 67.66$\pm$1.84 & 90.64$\pm$2.34 & 91.40$\pm$0.90 & 65.91$\pm$2.43 & 45.71$\pm$1.76 &4.83\\
DeProp-GCN & 79.52$\pm$2.63 & 67.24$\pm$2.02 & \underline{90.98$\pm$2.04} & 91.38$\pm$1.33 & 61.51$\pm$1.40 & 44.07$\pm$1.07 &5.17\\
PSNR-GCN & \textbf{81.01$\pm$1.63 }& \textbf{68.06$\pm$2.12} & \textbf{91.23$\pm$1.00} & \textbf{91.51$\pm$0.69} & \textbf{71.51$\pm$1.90} & \textbf{54.95$\pm$1.73} &\textbf{1.00}\\
\midrule
GAT &79.21$\pm$1.31 &67.12$\pm$1.59 &89.43$\pm$1.62 &89.64$\pm$0.84 &68.04$\pm$1.36 &47.93$\pm$1.99 &6.00\\
Res-GAT & 79.72$\pm$1.94 & 67.19$\pm$1.91 & 89.92$\pm$0.86 & 91.40$\pm$0.74 & \textbf{72.66$\pm$0.94} & \underline{55.98$\pm$2.12} &\underline{2.83}\\ 
Init-Res-GAT &79.67$\pm$1.71 & 66.84$\pm$2.52 & 89.61$\pm$1.02 & \underline{91.53$\pm$1.06} &69.89$\pm$1.81 & 51.29$\pm$1.42 &4.17\\
JK-GAT &80.04$\pm$1.61 &65.83$\pm$2.21 &90.10$\pm$0.94  &90.82$\pm$1.24 &67.98$\pm$1.71 &50.43$\pm$1.45 &4.50\\
Dense-GAT & 73.39$\pm$1.52 & 61.23$\pm$2.53 & 89.72$\pm$0.86 & 88.92$\pm$1.66 & 67.36$\pm$1.95 & 50.25$\pm$0.88 &7.00\\ 
DropMessage-GAT & \underline{80.36$\pm$1.92} & \underline{67.82$\pm$2.04} & \textbf{90.67$\pm$1.68 }& 90.98$\pm$0.90 & 63.23$\pm$2.23 & 45.23$\pm$1.35 &4.33\\ 
DeProp-GAT & 76.00$\pm$2.08 & 61.16$\pm$3.50 & 87.34$\pm$1.42 & 89.76$\pm$1.52 & 64.23$\pm$3.22 & 46.29$\pm$3.24 &8.00\\ 
Half-Hop-GAT &77.24$\pm$1.69 & 66.74$\pm$1.98 & 89.66$\pm$1.45 &89.92$\pm$0.77 &62.86$\pm$2.04 & 47.84$\pm$4.36 &6.83\\
PSNR-GAT & \textbf{80.47$\pm$1.62} & \textbf{68.01$\pm$2.14} &\underline{90.38$\pm$1.21} & \textbf{91.64$\pm$0.61} & \underline{72.24$\pm$1.69} & \textbf{60.85$\pm$1.61} &\textbf{1.33}\\ 
\hline
\hline
\end{tabular}
}
\vspace{-0.3cm}
\end{table} 

As can be observed from the Table \ref{ssnc}, PSNR outperforms all baselines in most cases. For example, compared to the vanilla model, PSNR improves the test accuracy on the Squirrel dataset by 7.52\% and 12.92\%, respectively. Compared with the vanilla GCN and GAT, the proposed PSNR can significantly improve the performance under the fully observed feature setting.

\begin{table}
  \caption{Test accuracy (\%) on missing feature setting. The best results are in bold and the second best results are underlined.}
  \label{ssnc-mv}
  \centering
  \resizebox{0.8\linewidth}{!}{
  \begin{tabular}{ccccccccccccc}
  \hline
\hline
      &\multicolumn{6}{c}{\textbf {GCN}} &\multicolumn{6}{c}{\textbf {GAT}} \\
      \cmidrule(r){1-1}
      \cmidrule(r){2-7}
      \cmidrule(r){8-13}
      \multirow{2}{*}{\textbf {Module}} &\multicolumn{2}{c}{\textbf{Cora}} & \multicolumn{2}{c}{\textbf{Citeseer}} & \multicolumn{2}{c}{\textbf{Pubmed}}&\multicolumn{2}{c}{\textbf{Cora}} & \multicolumn{2}{c}{\textbf{Citeseer}} & \multicolumn{2}{c}{\textbf{Pubmed}} \\
      \cmidrule(r){2-3}
      \cmidrule(r){4-5}
      \cmidrule(r){6-7}
      \cmidrule(r){8-9}
      \cmidrule(r){10-11}
      \cmidrule(r){12-13}
      &\textbf{Acc}&\#K & \textbf{Acc}& \#K & \textbf{Acc} &\#K&\textbf{Acc}&\#K & \textbf{Acc}& \#K & \textbf{Acc} &\#K\\
      \midrule 
      None &57.3 &3&44.0  &6&36.4 &4&50.1  & 2&40.8  & 4&38.5  &4\\
      BatchNorm &71.8 &20&45.1 &25&70.4 &30&72.7  & 5&48.7  & 5&60.7  &4\\
      PairNorm &65.6  & 20&43.6 &25&63.1 &30&68.8 & 8&50.3  & 6&63.2  &20\\
      
      DGN &76.3 & 20 &50.2 & 30&72.0 &30& 75.8 & 8 &54.5 & 5&72.3  &20\\
      DeCorr &73.8  & 20&49.1  & 30&73.3  &15&72.8  & 15&46.5  & 6&72.4  &15\\
      DropEdge &67.0  & 6&44.2  & 8&69.3  &6&67.2  & 6&48.2  & 6&67.2  &6\\
      Res & \underline{76.8} & \underline{8} & 60.4 & 10 & 76.6 & 6 &76.5 &8 &60.5 & 6 &\underline{76.9} &\underline{8}\\
      Init-Res &65.1 & 6 & 50.7 &15 &70.4 & 10 & \underline{77.1}&\underline{8}&60.6 &8 & 76.8 & 6 \\
      Dense & 66.2 & 4 & 51.5 & 2 &74.1 &8 & 68.5 &10 & 52.7 & 2 &75.1 &10\\
      JK & 75.5 & 30 &60.4&8&\underline{76.9} & \underline{6}
      &77.0& 10&60.3&4&76.8 &6\\
      DropMessage & 75.5 & 10 & \underline{61.0}& \underline{6}& 74.6 & 6 & 76.5 & 6 & \underline{61.1}& \underline{8}& 76.6 & 6 \\
      DeProp & 71.4 & 6 &	59.4 & 2 &	76.1 & 4 &	68.04 & 2 & 48.3 & 2	& 75.8& 4 \\
      Half-Hop & 73.7 & 8 &	59.48 & 6	& 76.5 & 6	& 76.0 & 20 &	59.6 & 4 &	\underline{76.9}& \underline{6}\\
      PSNR &\textbf{77.3} & \textbf{20}&\textbf{61.1} & \textbf{15}&\textbf{77.0} &\textbf{30}&\textbf{77.9} & \textbf{8}&\textbf{61.9} & \textbf{15}&\textbf{77.3} &\textbf{10}\\
      \hline
\hline
  \end{tabular}
  }
  \vspace{-0.3cm}
\end{table}
\subsection{Missing feature setting (RQ3)}
\label{exp:missing}
When do we need the deeper GNN? Real-world data often contain missing features. In that scenario, previous research \cite{Zhao_2019} has shown that deep GNNs can help improve performance. 
For the nodes with missing features, due to the lack of information, they need a deeper network to gather more neighborhood information, thereby achieving better node representations. However, deep GNNs face the issue of performance degradation. In this section, we examine if PSNR module can improve the performance of GNNs in the context of missing feature.

Consistent with \cite{Zhao_2019, DBLP:conf/kdd/LiuHJLL23, DBLP:conf/kdd/JinL0AT22}, we evaluate the performance of GNNs on three datasets, Cora, Citeseer, and Pubmed, and remove their node features from validation and test sets. Under this setting, the test nodes need more propagation layers to reach the training nodes. We reuse the metrics that already reported in \cite{DBLP:conf/kdd/JinL0AT22} for None, BatchNorm \cite{DBLP:conf/icml/IoffeS15}, PairNorm \cite{Zhao_2019}, DGN \cite{Zhou_2020}, DeCorr \cite{DBLP:conf/kdd/JinL0AT22}, and DropEdge \cite{Rong_2019}. For all residual-based models, the results are obtained by varying the number of layers in $\{2, 4, 6, 8, 10, 15, 20, 30\}$ and running five times for each number of layers. We select the layer \#K that achieves the best performance and report its average accuracy. The results are reported in Table~\ref{ssnc-mv}. By examining the results in Table~\ref{ssnc-mv}, under the  missing feature setting, the optimal number of layers to achieve the best performance is significantly higher than in the fully observed feature setting, demonstrating the importance of deep GNNs. And PSNR outperform other baselines in all cases
through alleviating over-smoothing more effectively. Specially, on the Pubmed dataset, PSNR boost the accuracy of GCN and GAT by 40.6\% and 38.8\%, respectively.

\subsection{Performance on large graphs (RQ4)}
\label{ls}
To validate the scalability of PSNR, we conducted additional experiments on three larger graph datasets i.e.,
Coauthor-Physics, Flickr and Ogbn-arxiv, to further validate the effectiveness and scalability of our method. Specifically, we selected the GCN backbone for our experiment. 
We report the performance of GCN and various residual methods on three datasets and the memory consumption on the largest dataset, Ogbn-arxiv. The experimental results are presented in Table \ref{large}, from which we observe that PSNR-GCN scales well and achieves the best results across all three large datasets. Meanwhile, in terms of memory consumption, PSNR is slightly higher than GCN and ResGCN, comparable to GCNII with the same initial residuals, and significantly lower than JKNet and DenseGCN. Regarding training time, PSNR is roughly at the same level as JKNet, and its time is shorter than that of DenseNet.

\begin{table}[h]
\vspace{-0.3cm}
\caption{Comparison of different methods across various datasets and memory consumption (MB) and training time (ms / epoch) on Ogbn-arxiv. The best performance for each dataset is in bold, while the second best is underlined.}
\centering
\label{large}
\resizebox{0.8\linewidth}{!}{
\begin{tabular}{lccc|cl}
\hline
\hline
Method & Phy & Flickr & Ogbn-arxiv & Memory  &Time \\
\midrule
GCN        & 95.32 $\pm$ 0.11 & 51.40 $\pm$ 0.33 & 64.37 $\pm$ 0.41 &2421 & 30.10\\
ResGCN        & 95.61 $\pm$ 0.18 & 51.90 $\pm$ 0.16 & \underline{66.32 $\pm$ 0.59} & 2463  & 36.06\\
GCNII    & \underline{95.90 $\pm$ 0.14} & 46.18 $\pm$ 0.21 & 61.43 $\pm$ 1.63 & 2525 & 33.33\\
JKNet      & 95.88 $\pm$ 0.15 & 51.65 $\pm$ 0.31 & 60.46 $\pm$ 1.21 & 2921  & 40.05\\
DenseGCN      & 95.50 $\pm$ 0.12 & \underline{52.18 $\pm$ 0.25} & 62.46 $\pm$ 1.58 & 3131  & 52.24\\
PSNR-GCN  & \textbf{95.92 $\pm$ 0.17} & \textbf{52.47 $\pm$ 0.16} & \textbf{67.81 $\pm$ 0.57}  & 2539 & 42.93\\
\hline
\hline
\end{tabular}}
\label{tab:comparison}
\vspace{-0.3cm}
\end{table}

\vspace{-0.3cm}
\section{Conclusion and Future Work}
\label{sec: conclusion}
In this paper, we addressed the oversmoothing in Graph Neural Networks (GNNs) with a focus on residual methods. We revisit the oversmoothing from the perspective of overlapping neighborhood subgraphs, explaining why residual methods can alleviate it. Our analysis revealed that current residual methods often lack node adaptivity and struggle with information loss in high-order neighborhoods subgraphs. To overcome these limitations, we introduce the Posteriori-Sampling-based Node-Adaptive Residual Module (PSNR). This innovative module uses a graph encoder to learn the posterior distribution of residual coefficients for each node at different layers, enabling fine-grained, node-adaptive neighborhood subgraph aggregation with minimal overhead. Extensive experiments confirmed that the PSNR module can effectively mitigate oversmoothing and improve performance, particularly in scenarios requiring deep networks. Despite the significant progress made by PSNR, training a deeper GNN remains challenging, prompting the need for further research in the future. 


\section{Acknowledgements}
This work was supported by National Science and Technology Major Project (No. 2022ZD0115101), National Natural Science Foundation of China Project (No. 623B2086, No. U21A20427), Project (No. WU2022A009) from the Center of Synthetic Biology and Integrated Bioengineering of Westlake University and Integrated Bioengineering of Westlake University and Project (No. WU2023C019) from the Westlake University Industries of the Future Research Funding. Carl Yang was not supported by any funds from China.


\appendix

\newpage
\section{Fully Observed Node Classification}
\label{app_a}
\vspace{1.5cm}
\begin{table}[ht]
\centering
\caption{Node classification accuracy (\%) on GCN backbone. The best results across different layers are highlighted.}
\begin{adjustbox}{max width=\textwidth}
\begin{tabular}{lllllllllll}
\hline
\hline
\textbf{Datasets} & \textbf{Model} & \textbf{Layer 2} & \textbf{Layer 4} & \textbf{Layer 8} & \textbf{Layer 16} & \textbf{Layer 32} & \textbf{Layer 64} \\ 
\midrule
\multirow{7}{*}{Cora} 
& GCN & \cellcolor[gray]{0.8}78.56$\pm$1.57 & 78.52$\pm$2.02 & 77.27$\pm$2.29 & 72.11$\pm$6.30 & 46.09$\pm$12.12 & 30.76$\pm$1.30 \\ 
& ResGCN & 79.39$\pm$1.12 & 79.80$\pm$1.53 & 80.10$\pm$1.46 & \cellcolor[gray]{0.8}80.11$\pm$0.98 & 71.27$\pm$3.65 & 36.48$\pm$7.53 \\ 
& GCNII & \cellcolor[gray]{0.8}77.40$\pm$1.43 & 76.47$\pm$1.96 & 76.39$\pm$1.88 & 76.15$\pm$1.70 & 75.76$\pm$1.83 & 76.27$\pm$1.46 \\ 
& DenseGCN & \cellcolor[gray]{0.8}74.94$\pm$2.46 & 67.06$\pm$2.02 & 67.58$\pm$2.66 & 65.08$\pm$3.71 & 59.98$\pm$2.30 & 34.37$\pm$5.69 \\ \
& JKNet & 78.44$\pm$2.10 & \cellcolor[gray]{0.8}79.42$\pm$1.48 & 79.36$\pm$2.23 & 77.18$\pm$2.17 & 77.33$\pm$2.88 & 51.58$\pm$7.29 \\ 
& DropMessage-GCN & \cellcolor[gray]{0.8}80.00$\pm$2.22 & 79.88$\pm$1.81 & 78.14$\pm$1.65 & 76.40$\pm$1.22 & 52.18$\pm$1.06 & 30.34$\pm$0.63 \\ 
& DeProp-GCN & \cellcolor[gray]{0.8}79.52$\pm$2.63 & 78.44$\pm$2.39 & 33.02$\pm$5.64 & 31.28$\pm$3.24 & 30.20$\pm$0.01 & 30.20$\pm$0.01 \\ 
& Half-Hop-GCN & 78.92$\pm$2.08 & \cellcolor[gray]{0.8}80.74$\pm$1.80 & 80.24$\pm$1.61 & 46.66$\pm$13.37 & 43.68$\pm$7.14 & 31.03$\pm$7.39 \\ 
& PSNR-GCN & 80.59$\pm$1.57 & \cellcolor[gray]{0.8}81.01$\pm$1.63 & 80.55$\pm$1.57 & 78.26$\pm$1.36 & 76.34$\pm$2.18 & 77.75$\pm$2.27 \\ 
\midrule
\multirow{7}{*}{Citeseer} 
& GCN & \cellcolor[gray]{0.8}66.00$\pm$2.37 & 64.13$\pm$2.31 & 64.13$\pm$2.10 & 58.52$\pm$3.31 & 27.21$\pm$3.98 & 27.52$\pm$5.04 \\ 
& Res-GCN & 66.11$\pm$1.65 & \cellcolor[gray]{0.8}66.40$\pm$2.29 & 65.97$\pm$1.93 & 65.46$\pm$2.02 & 48.70$\pm$3.77 & 33.74$\pm$4.79 \\ 
& GCNII & \cellcolor[gray]{0.8}66.69$\pm$1.95 & 66.18$\pm$1.74 & 66.50$\pm$1.77 & 66.31$\pm$2.08 & 66.27$\pm$1.92 & 66.50$\pm$2.59 \\ 
& Dense-GCN & \cellcolor[gray]{0.8}62.54$\pm$2.56 & 58.99$\pm$3.74 & 54.35$\pm$5.42 & 50.10$\pm$5.24 & 49.09$\pm$4.03 & 31.43$\pm$5.54 \\ 
& JKNet & \cellcolor[gray]{0.8}65.49$\pm$2.32 & 64.40$\pm$2.28 & 64.18$\pm$3.07 & 63.77$\pm$1.87 & 60.88$\pm$3.53 & 29.66$\pm$7.51 \\ 
& DropMessage-GCN &\cellcolor[gray]{0.8}67.66$\pm$1.84 & 67.04$\pm$2.20 & 63.84$\pm$2.66 & 59.82$\pm$2.75 & 22.68$\pm$3.77 & 21.38$\pm$1.72 \\ 
& Half-Hop-GCN & \cellcolor[gray]{0.8}67.76$\pm$2.19 & 67.46$\pm$1.99 & 67.02$\pm$2.53 & 54.48$\pm$3.93 & 38.40$\pm$8.33 & 24.14$\pm$3.48 \\ 
& DeProp-GCN & \cellcolor[gray]{0.8}67.24$\pm$2.02 & 64.86$\pm$2.93 & 29.08$\pm$10.03 & 21.72$\pm$0.91 & 21.94$\pm$1.18 & 20.98$\pm$1.13 \\ 
& PSNR-GCN & \cellcolor[gray]{0.8}68.06$\pm$2.12 & 66.03$\pm$1.93 & 65.46$\pm$1.59 & 65.81$\pm$2.39 & 65.52$\pm$1.76 & 65.85$\pm$2.30 \\  
\midrule
\multirow{7}{*}{CS} 
& GCN & \cellcolor[gray]{0.8}90.19$\pm$0.83 & 88.81$\pm$0.62 & 86.93$\pm$0.83 & 84.47$\pm$0.95 & 71.40$\pm$4.93 & 36.74$\pm$4.57 \\ 
& Res-GCN & \cellcolor[gray]{0.8}90.86$\pm$0.99 & 90.63$\pm$0.96 & 89.97$\pm$0.90 & 88.40$\pm$0.90 & 85.01$\pm$1.53 & 64.39$\pm$3.86 \\ 
& GCNII & \cellcolor[gray]{0.8}90.34$\pm$0.78 & 90.02$\pm$0.81 & 90.07$\pm$0.77 & 90.09$\pm$0.94 & 90.17$\pm$0.52 & 89.93$\pm$0.70 \\ 
& Dense-GCN & 89.01$\pm$1.09 & 89.99$\pm$1.27 & \cellcolor[gray]{0.8}90.28$\pm$1.06 & 89.33$\pm$1.73 & 88.92$\pm$1.08 & 88.40$\pm$1.37 \\ 
& JKNet & 90.81$\pm$1.35 & \cellcolor[gray]{0.8}90.94$\pm$0.89 & 90.53$\pm$1.41 & 89.52$\pm$1.11 & 88.81$\pm$1.14 & 88.44$\pm$1.44 \\ 
& DropMessage-GCN & \cellcolor[gray]{0.8}90.64$\pm$2.34 & 89.18$\pm$1.76 & 89.38$\pm$1.52 & 85.62$\pm$7.01 & 87.94$\pm$6.18 & 84.44$\pm$1.13 \\ 
& Half-Hop-GCN & \cellcolor[gray]{0.8}90.30$\pm$1.83 & 89.30$\pm$2.23 & 88.98$\pm$1.81 & 82.98$\pm$1.88 & 54.60$\pm$4.55 & 25.36$\pm$10.27 \\ 
& DeProp-GCN & \cellcolor[gray]{0.8}90.98$\pm$2.04 & 89.10$\pm$2.23 & 71.24$\pm$4.18 & 72.34$\pm$2.26 & 52.60$\pm$0.03 & 22.60$\pm$0.01 \\ 
& PSNR-GCN & \cellcolor[gray]{0.8}91.23$\pm$1.00 & 90.70$\pm$1.49 & 90.26$\pm$1.17 & 90.26$\pm$0.98 & 90.52$\pm$1.02 & 90.30$\pm$0.88 \\  
\midrule
\multirow{7}{*}{Photo} 
& GCN & \cellcolor[gray]{0.8}91.22$\pm$0.73 & 90.49$\pm$0.76 & 88.10$\pm$1.02 & 80.05$\pm$4.25 & 54.80$\pm$8.00 & 46.57$\pm$9.64 \\ 
& Res-GCN & \cellcolor[gray]{0.8}91.31$\pm$0.85 & 90.97$\pm$0.78 & 90.71$\pm$0.78 & 85.98$\pm$2.36 & 65.42$\pm$6.75 & 56.15$\pm$10.43 \\ 
& GCNII & 91.02$\pm$0.93 & 90.98$\pm$0.92 & 91.02$\pm$0.70 & 90.99$\pm$0.76 & 91.20$\pm$0.81 & \cellcolor[gray]{0.8}91.46$\pm$0.84 \\ 
& Dense-GCN & \cellcolor[gray]{0.8}90.38$\pm$1.02 & 90.07$\pm$1.76 & 89.34$\pm$1.40 & 87.77$\pm$2.00 & 86.01$\pm$1.91 & 68.22$\pm$17.58 \\ 
& JKNet & 87.26$\pm$1.77 & 87.96$\pm$1.91 & 84.39$\pm$2.76 & \cellcolor[gray]{0.8}90.85$\pm$1.19 & 90.10$\pm$1.20 & 87.93$\pm$2.66 \\ 
& DropMessage-GCN & \cellcolor[gray]{0.8}91.40$\pm$0.90 & 90.22$\pm$1.19 & 87.57$\pm$2.74 & 87.82$\pm$1.21 & 86.12$\pm$1.05 & 80.40$\pm$1.07 \\ 
& Half-Hop-GCN & 52.00$\pm$17.28 & 67.48$\pm$24.16 & \cellcolor[gray]{0.8}89.32$\pm$1.17 & 83.66$\pm$3.02 & 64.40$\pm$4.30 & 39.00$\pm$9.42 \\ 
& DeProp-GCN & \cellcolor[gray]{0.8}91.38$\pm$1.33 & 89.50$\pm$5.31 & 78.12$\pm$3.66 & 81.86$\pm$9.36 & 87.74$\pm$4.19 & 84.43$\pm$2.31 \\ 
& PSNR-GCN & 91.44$\pm$0.82 & 91.20$\pm$1.03 & 91.39$\pm$0.76 & 91.11$\pm$0.68 & \cellcolor[gray]{0.8}91.51$\pm$0.69 & 91.49$\pm$0.88 \\ 
\midrule
\multirow{7}{*}{Chameleon} 
& GCN & \cellcolor[gray]{0.8}67.70$\pm$0.17 & 56.30$\pm$2.28 & 48.37$\pm$1.56 & 46.81$\pm$2.06 & 39.62$\pm$2.13 & 33.27$\pm$1.86 \\ 
& Res-GCN & 70.28$\pm$1.14 & \cellcolor[gray]{0.8}70.70$\pm$1.54 & 69.18$\pm$0.78 & 51.64$\pm$2.11 & 42.61$\pm$3.60 & 31.42$\pm$2.73 \\ 
& GCNII & 69.45$\pm$1.63 & 69.36$\pm$1.13 & 69.62$\pm$1.61 & 69.23$\pm$1.21 & 68.87$\pm$2.16 & \cellcolor[gray]{0.8}69.69$\pm$1.36 \\ 
& Dense-GCN & \cellcolor[gray]{0.8}67.82$\pm$1.27 & 64.59$\pm$1.29 & 65.36$\pm$1.50 & 64.39$\pm$1.37 & 30.63$\pm$3.98 & 28.72$\pm$4.66 \\ 
& JKNet & \cellcolor[gray]{0.8}68.26$\pm$1.15 & 67.97$\pm$1.59 & 66.61$\pm$1.76 & 67.78$\pm$1.60 & 66.46$\pm$2.25 & 52.98$\pm$2.79 \\ 
& DropMessage-GCN & \cellcolor[gray]{0.8}65.91$\pm$2.43 & 60.50$\pm$2.84 & 47.29$\pm$2.00 & 42.66$\pm$1.60 & 34.56$\pm$1.61 & 31.74$\pm$1.56 \\ 
& Half-Hop-GCN & \cellcolor[gray]{0.8}60.07$\pm$2.34 & 55.72$\pm$3.12 & 52.64$\pm$1.99 & 43.19$\pm$2.55 & 36.26$\pm$2.12 & 29.34$\pm$2.20 \\ 
& DeProp-GCN & \cellcolor[gray]{0.8}61.51$\pm$1.40 & 58.25$\pm$2.22 & 46.46$\pm$2.18 & 31.88$\pm$2.52 & 30.81$\pm$2.07 & 33.36$\pm$0.96 \\ 
& PSNR-GCN & \cellcolor[gray]{0.8}71.51$\pm$1.90 & 70.74$\pm$2.24 & 71.40$\pm$1.83 & 66.00$\pm$1.97 & 65.82$\pm$2.02 & 69.67$\pm$4.59 \\ 
\midrule
\multirow{7}{*}{Squirrel} 
& GCN & \cellcolor[gray]{0.8}47.43$\pm$1.31 & 41.68$\pm$1.80 & 37.43$\pm$1.35 & 33.17$\pm$1.22 & 31.99$\pm$1.14 & 29.83$\pm$1.98 \\ 
& Res-GCN & 51.72$\pm$1.63 & \cellcolor[gray]{0.8}52.43$\pm$1.72 & 50.38$\pm$1.90 & 38.35$\pm$1.56 & 26.30$\pm$2.06 & 22.66$\pm$1.01 \\ 
& GCNII & \cellcolor[gray]{0.8}46.13$\pm$1.59 & 45.28$\pm$1.59 & 45.71$\pm$2.08 & 45.37$\pm$1.72 & 45.06$\pm$1.50 & 45.32$\pm$1.85 \\ 
& Dense-GCN & \cellcolor[gray]{0.8}49.57$\pm$1.58 & 49.01$\pm$0.98 & 49.51$\pm$1.30 & 49.08$\pm$1.36 & 49.29$\pm$1.25 & 49.04$\pm$1.30 \\ 
& JKNet & \cellcolor[gray]{0.8}49.87$\pm$1.44 & 49.11$\pm$2.27 & 46.62$\pm$1.44 & 24.79$\pm$3.88 & 45.73$\pm$1.80 & 41.75$\pm$1.96 \\ 
& DropMessage-GCN & \cellcolor[gray]{0.8}45.71$\pm$1.76 & 38.68$\pm$2.02 & 26.67$\pm$3.06 & 25.71$\pm$1.15 & 23.77$\pm$1.18 & 20.06$\pm$0.23 \\ 
& Half-Hop-GCN & \cellcolor[gray]{0.8}42.78$\pm$1.01 & 41.87$\pm$1.36 & 37.84$\pm$1.68 & 31.60$\pm$2.12 & 26.36$\pm$1.66 & 21.07$\pm$1.53 \\ 
& DeProp-GCN & \cellcolor[gray]{0.8}44.07$\pm$1.07 & 40.98$\pm$1.29 & 33.38$\pm$1.70 & 28.33$\pm$0.96 & 28.42$\pm$0.54 & 29.98$\pm$2.78 \\ 
& PSNR-GCN & \cellcolor[gray]{0.8}54.95$\pm$1.73 & 54.13$\pm$1.41 & 50.68$\pm$1.45 & 50.22$\pm$6.32 & 50.06$\pm$0.97 & 50.24$\pm$1.58 \\  
\hline
\hline
\end{tabular}
\end{adjustbox}
\label{tab:hyperparameters}
\end{table}

\newpage
\newpage

\vspace{4cm}

\begin{table}[ht]
\centering
\caption{Node classification accuracy (\%) on GAT backbone. The best results across different layers are highlighted.}
\begin{adjustbox}{max width=\textwidth}
\begin{tabular}{lllllllllll}
\hline
\hline
\textbf{Datasets} & \textbf{Model} & \textbf{Layer 2} & \textbf{Layer 4} & \textbf{Layer 8} & \textbf{Layer 16} & \textbf{Layer 32} & \textbf{Layer 64} \\  
\midrule
\multirow{7}{*}{Cora} 
& GAT & \cellcolor[gray]{0.8}79.21$\pm$1.31 & 78.83$\pm$1.78 & 41.48$\pm$2.84 & 30.20$\pm$0.00 & 30.20$\pm$0.00 & 30.20$\pm$0.00 \\ 
& Res-GAT & 79.42$\pm$1.46 & \cellcolor[gray]{0.8}79.72$\pm$1.94 & 78.88$\pm$1.69 & 78.71$\pm$2.13 & 77.80$\pm$2.25 & 30.22$\pm$0.12 \\
& InitRes-GAT & 79.32$\pm$1.65 & 79.39$\pm$1.45 & 79.00$\pm$1.97 & 79.16$\pm$1.68 & \cellcolor[gray]{0.8}79.67$\pm$1.71 & 79.42$\pm$1.97 \\
& Dense-GAT & \cellcolor[gray]{0.8}73.39$\pm$1.52 & 69.07$\pm$2.76 & 58.85$\pm$3.02  &60.34$\pm$2.12 & 59.44$\pm$3.03 & 55.80$\pm$3.12\\
& JK-GAT & 79.32$\pm$1.06 & \cellcolor[gray]{0.8}80.04$\pm$1.61 & 77.28$\pm$1.62 & 77.58$\pm$1.51 & 77.53$\pm$1.29 & 77.58$\pm$2.91 \\
& DropMessage-GAT & 80.06$\pm$2.50 & \cellcolor[gray]{0.8}80.36$\pm$1.92 & 79.20$\pm$0.22 & 76.54$\pm$2.49 & 52.54$\pm$1.02 & 30.17$\pm$0.06 \\
& DeProp-GAT & \cellcolor[gray]{0.8}76.00$\pm$2.08 & 69.06$\pm$3.94 & 30.22$\pm$0.06 & 31.20$\pm$0.02 & 30.20$\pm$0.01 & 30.20$\pm$0.01 \\
& Half-Hop-GAT & 77.20$\pm$1.86 & 76.98$\pm$1.42 & \cellcolor[gray]{0.8}77.24$\pm$1.69 & 75.84$\pm$1.68 & 70.50$\pm$4.69 & 30.20$\pm$0.10 \\
& PSNR-GAT & \cellcolor[gray]{0.8}80.47$\pm$1.62 & 80.22$\pm$0.98 & 79.96$\pm$1.69 & 80.02$\pm$1.64 & 79.43$\pm$1.78 & 79.69$\pm$1.29 \\ 
\midrule
\multirow{7}{*}{Citeseer} 
& GAT & \cellcolor[gray]{0.8}67.12$\pm$1.59 & 65.14$\pm$3.02 & 21.90$\pm$2.40 & 21.10$\pm$0.00 & 21.10$\pm$0.00 & 21.10$\pm$0.00 \\
& Res-GAT & \cellcolor[gray]{0.8}67.19$\pm$1.91 & 65.65$\pm$2.21 & 63.26$\pm$2.31 & 63.47$\pm$1.18 & 62.82$\pm$2.92 & 22.66$\pm$1.31 \\
& InitRes-GAT & \cellcolor[gray]{0.8}66.84$\pm$2.52 & 64.15$\pm$2.44 & 65.80$\pm$1.89 & 64.89$\pm$2.65 & 64.69$\pm$2.49 & 65.11$\pm$2.65 \\
& Dense-GAT & \cellcolor[gray]{0.8}61.63$\pm$2.53 & 54.84$\pm$2.34 & 50.69$\pm$3.79 & 49.21$\pm$3.48 & 47.24$\pm$3.08 & 46.34$\pm$3.62 \\
& JK-GAT & \cellcolor[gray]{0.8}65.83$\pm$2.21 & 64.79$\pm$2.46 & 63.44$\pm$2.51 & 65.00$\pm$1.43 & 63.57$\pm$2.27 & 62.08$\pm$2.26 \\
& DropMessage-GAT & \cellcolor[gray]{0.8}67.82$\pm$2.04 & 66.64$\pm$2.51 & 65.10$\pm$3.32 & 58.20$\pm$2.55 & 21.90$\pm$1.64 & 20.90$\pm$2.04 \\
& DeProp-GAT & \cellcolor[gray]{0.8}61.16$\pm$3.50 & 51.52$\pm$11.25 & 23.20$\pm$2.77 & 21.12$\pm$1.76 & 21.00$\pm$1.67 & 20.50$\pm$0.92 \\
& Half-Hop-GAT & \cellcolor[gray]{0.8}66.74$\pm$1.98 & 66.70$\pm$2.58 & 63.12$\pm$3.49 & 60.00$\pm$3.30 & 45.98$\pm$4.63 & 24.32$\pm$4.46 \\
& PSNR-GAT & \cellcolor[gray]{0.8}68.01$\pm$2.14 & 65.61$\pm$2.05 & 66.46$\pm$2.14 & 65.48$\pm$2.59 & 66.36$\pm$2.25 & 65.50$\pm$2.35 \\ 
\midrule
\multirow{7}{*}{CS} 
& GAT & \cellcolor[gray]{0.8}89.43$\pm$1.62 & 77.23$\pm$10.14 & 72.63$\pm$5.30 & 38.81$\pm$6.52 & 22.60$\pm$0.00 & 22.60$\pm$0.00 \\
& Res-GAT & 89.48$\pm$1.25 & 83.48$\pm$3.39 & \cellcolor[gray]{0.8}89.92$\pm$0.86 & 64.30$\pm$19.93 & 23.10$\pm$1.45 & 22.46$\pm$2.22 \\
& InitRes-GAT & \cellcolor[gray]{0.8}89.61$\pm$1.02 & 77.82$\pm$5.90 & 89.43$\pm$1.23 & 88.50$\pm$1.17 & 88.20$\pm$1.22 & 86.90$\pm$2.30 \\
& Dense-GAT & \cellcolor[gray]{0.8}89.72$\pm$0.86 & 88.81$\pm$0.89 &88.28$\pm$0.72 ~ &87.01$\pm$1.52 ~ & 86.47$\pm$1.13 ~ & 87.23$\pm$1.78~ \\
& JK-GAT & 89.77$\pm$1.02 & \cellcolor[gray]{0.8}90.10$\pm$0.94 & 90.03$\pm$1.26 & 89.68$\pm$1.41 & 89.66$\pm$1.04 & 89.86$\pm$1.12 \\
& DropMessage-GAT & \cellcolor[gray]{0.8}90.67$\pm$1.68 & 89.20$\pm$1.58 & 87.58$\pm$1.46 & 85.68$\pm$1.29 & 84.58$\pm$1.40 & 81.43$\pm$1.13 \\
& DeProp-GAT & \cellcolor[gray]{0.8}87.34$\pm$1.42 & 84.18$\pm$1.99 & 72.26$\pm$2.77 & 72.36$\pm$2.34 & 52.60$\pm$0.02 & 22.60$\pm$0.01 \\
& Half-Hop-GAT & \cellcolor[gray]{0.8}89.66$\pm$1.45 & 89.30$\pm$1.84 & 84.12$\pm$2.27 & 84.75$\pm$2.96 & 56.20$\pm$5.96 & 29.61$\pm$9.70 \\
& PSNR-GAT & \cellcolor[gray]{0.8}90.38$\pm$1.12 & 86.31$\pm$2.68 & 84.48$\pm$1.24 & 83.84$\pm$1.38 & 81.85$\pm$3.55 & 79.85$\pm$5.24 \\
\midrule
\multirow{7}{*}{Photo} 
& GAT & \cellcolor[gray]{0.8}89.64$\pm$0.84 & 42.4$\pm$18.86 & 51.13$\pm$10.00 & 27.97$\pm$7.71 & 25.40$\pm$0.00 & 25.40$\pm$0.00 \\
& Res-GAT & 89.03$\pm$1.33 & 89.30$\pm$4.68 & \cellcolor[gray]{0.8}91.40$\pm$0.74 & 88.60$\pm$1.79 & 25.61$\pm$4.23 & 26.16$\pm$1.88 \\
& InitRes-GAT & 88.92$\pm$2.22 & 31.74$\pm$3.12 & 91.23$\pm$1.22 & 90.72$\pm$1.40 & \cellcolor[gray]{0.8}91.53$\pm$1.06 & 91.11$\pm$1.21 \\
& Dense-GAT & \cellcolor[gray]{0.8}88.92$\pm$1.66 & 88.10$\pm$1.50 &87.03$\pm$2.22 ~ &86.79$\pm$1.97 ~ &86.24$\pm$2.65 ~ & 85.63$\pm$2.15 \\
& JK-GAT & 87.55$\pm$1.39 & \cellcolor[gray]{0.8}90.82$\pm$1.24 & 90.35$\pm$1.61 & 90.30$\pm$1.37 & 90.29$\pm$1.49 & 90.07$\pm$1.48 \\
& DropMessage-GAT & \cellcolor[gray]{0.8}90.98$\pm$0.90 & 90.29$\pm$0.97 & 82.15$\pm$2.82 & 84.96$\pm$1.56 & 85.10$\pm$0.90 & 82.15$\pm$1.32 \\
& DeProp-GAT & \cellcolor[gray]{0.8}89.76$\pm$1.52 & 87.98$\pm$1.05 & 82.54$\pm$1.01 & 81.25$\pm$2.88 & 82.26$\pm$2.17 & 80.00$\pm$2.11 \\
& Half-Hop-GAT & \cellcolor[gray]{0.8}89.92$\pm$0.77 & 85.48$\pm$1.89 & 85.16$\pm$2.77 & 72.12$\pm$1.54 & 67.84$\pm$1.08 & 45.51$\pm$0.65 \\
& PSNR-GAT & 90.93$\pm$1.42 & 91.48$\pm$0.94 & 91.33$\pm$1.05 & 91.05$\pm$0.95 & \cellcolor[gray]{0.8}91.64$\pm$0.61 & 91.37$\pm$0.96 \\
\midrule
\multirow{7}{*}{Chameleon} 
& GAT & \cellcolor[gray]{0.8}68.04$\pm$1.36 & 48.37$\pm$3.71 & 26.86$\pm$6.17 & 22.86$\pm$0.00 & 22.86$\pm$0.00 & 22.86$\pm$0.00 \\
& Res-GAT & 70.90$\pm$1.34 & 67.87$\pm$3.43 & \cellcolor[gray]{0.8}72.66$\pm$0.94 & 66.68$\pm$1.13 & 27.54$\pm$3.63 & 27.27$\pm$3.44 \\
& InitRes-GAT & 69.87$\pm$1.65 & 61.63$\pm$5.02 & \cellcolor[gray]{0.8}69.89$\pm$1.81 & 69.34$\pm$1.34 & 68.62$\pm$1.32 & 68.53$\pm$1.23 \\
& Dense-GAT &\cellcolor[gray]{0.8}67.36$\pm$1.95  & 66.87$\pm$2.07 & 66.21$\pm$2.03 & 61.95$\pm$2.14 & 62.32$\pm$2.10 & 61.71$\pm$1.69  \\
& JK-GAT & \cellcolor[gray]{0.8}67.98$\pm$1.71 & 65.71$\pm$1.40 & 66.75$\pm$1.88 & 66.99$\pm$2.23 & 66.24$\pm$1.42 & 66.31$\pm$1.56 \\
& DropMessage-GAT & \cellcolor[gray]{0.8}63.23$\pm$2.23 & 56.98$\pm$2.54 & 48.53$\pm$2.95 & 47.00$\pm$1.93 & 40.19$\pm$1.82 & 33.75$\pm$2.23 \\
& DeProp-GAT & \cellcolor[gray]{0.8}64.23$\pm$3.22 & 53.31$\pm$4.19 & 42.29$\pm$2.77 & 41.19$\pm$2.17 & 32.29$\pm$2.77 & 32.26$\pm$0.85 \\
& Half-Hop-GAT & \cellcolor[gray]{0.8}62.86$\pm$2.04 & 60.34$\pm$2.01 & 50.58$\pm$7.19 & 43.05$\pm$1.74 & 33.75$\pm$2.37 & 32.27$\pm$1.42 \\
& PSNR-GAT & 71.29$\pm$1.64 & 72.04$\pm$1.82 & 71.58$\pm$1.99 & 71.78$\pm$1.51 & 71.47$\pm$2.54 & \cellcolor[gray]{0.8}72.24$\pm$1.69 \\
\midrule
\multirow{7}{*}{Squirrel} 
& GAT & \cellcolor[gray]{0.8}47.93$\pm$1.99 & 32.96$\pm$1.85 & 20.57$\pm$1.02 & 20.00$\pm$0.00 & 20.00$\pm$0.00 & 20.00$\pm$0.00 \\
& Res-GAT & 52.87$\pm$1.68 & 49.87$\pm$4.74 & \cellcolor[gray]{0.8}55.98$\pm$2.12 & 50.31$\pm$1.68 & 23.07$\pm$1.98 & 22.01$\pm$1.26 \\
& InitRes-GAT & \cellcolor[gray]{0.8}51.29$\pm$1.42 & 43.92$\pm$5.77 & 49.83$\pm$1.55 & 50.34$\pm$1.17 & 50.30$\pm$1.81 & 50.46$\pm$1.75 \\
& Dense-GAT & \cellcolor[gray]{0.8}50.25$\pm$0.88 & 49.57$\pm$1.80 & 46.71$\pm$1.69 & 47.83$\pm$1.70 & 47.54$\pm$1.58 & 46.87$\pm$1.65\\
& JK-GAT & \cellcolor[gray]{0.8}50.43$\pm$1.45 & 44.41$\pm$2.33 & 49.08$\pm$0.79 & 49.31$\pm$1.99 & 48.87$\pm$1.74 & 49.56$\pm$1.31 \\
& DropMessage-GAT & \cellcolor[gray]{0.8}45.23$\pm$1.35 & 40.31$\pm$1.25 & 30.69$\pm$1.84 & 29.44$\pm$1.34 & 28.35$\pm$2.50 & 25.47$\pm$1.33 \\
& DeProp-GAT & 45.76$\pm$1.32& \cellcolor[gray]{0.8}46.29$\pm$3.24 & 29.99$\pm$0.28 & 28.97$\pm$0.56 & 28.91$\pm$0.23 & 20.00$\pm$0.04 \\
& Half-Hop-GAT & 43.46$\pm$1.57 & \cellcolor[gray]{0.8}47.84$\pm$4.36 & 42.53$\pm$1.53 & 29.19$\pm$1.50 & 24.37$\pm$1.54 & 23.79$\pm$0.65 \\
& PSNR-GAT & 57.81$\pm$2.08 & \cellcolor[gray]{0.8}60.85$\pm$1.61 & 59.58$\pm$2.09 & 60.43$\pm$2.20 & 60.00$\pm$2.20 & 60.20$\pm$1.53 \\  
\hline
\hline
\end{tabular}
\end{adjustbox}
\label{tab:hyperparameters}
\end{table}

\newpage
\section{Residual Connection Defination} 
\label{app_b}
Common residual connection for GNNs and their corresponding GNNs are described below.

\textbf{Res.} Res is composed of multiple residual blocks containing few stacked layers. Taking the initial input of the $n$-th residual block as $\mathbf{X}_n$, and the stacked nonlinear layers within the residual block as $\mathbf{F(X)}$:
\begin{displaymath}
    \mathbf{X}_{n+1}=\mathbf{F}(\mathbf{X}_n)+\mathbf{X}_n,
\end{displaymath}
where residual mapping and identity mapping refer to $\mathbf{F}(\mathbf{X})$ and $\mathbf{X}$ on the right side of the above equation, respectively. Inspired by Res, Guohao Li \& Matthias Müller(2019) proposed a residual connection learning framework for GCN and called this model ResGCN which can be simply described as follows:
\begin{displaymath}
\mathbf{H}_k =  \sigma\left(\tilde{\mathbf{D}}^{-\frac{1}{2}} \tilde{\mathbf{A}} \tilde{\mathbf{D}}^{-\frac{1}{2}} \mathbf{H}_{k-1} \mathbf{W}_{k-1}\right) + \mathbf{H}_{k-1}.   
\end{displaymath}

\textbf{InitialRes.} InitialRes is proposed for the first time in APPNP, unlike Res that carries information from the previous layers, it constructs a connection to the initial representation $\mathbf{X}_0$ at each layer:
\begin{displaymath}
\mathbf{X}_{n+1}=(1-\alpha)\mathbf{H}(\mathbf{X}_n)+\alpha{\mathbf{X}_0},
\end{displaymath}
where $\mathbf{H}(\mathbf{X})$ denotes the aggregation operation within one layer. InitialRes ensures that each node's representation retains at least an $\alpha$-sized portion of the initial feature information. Correspondingly, APPNP can be formulated as: 
\begin{displaymath}
    \mathbf{H}_k = \left(1-\alpha\right)\tilde{\mathbf{D}}^{-\frac{1}{2}} \tilde{\mathbf{A}} \tilde{\mathbf{D}}^{-\frac{1}{2}} \mathbf{H}_{k-1} + \alpha\mathbf{H}.
\end{displaymath}
Based on APPNP, GCNII introduces identity mapping from Res to make up for the deficiency in APPNP.

\textbf{Dense.} Dense proposes a more efficient way to reuse features between layers. The input is the outputs of all previous layers of the network and at each layer Dense concats them together:
\begin{displaymath}
\mathbf{X}_{n+1}=\mathbf{H}([\mathbf{X}_0,\mathbf{X}_1, \ldots,\mathbf{X}_n]),
\end{displaymath}
where $[\cdot]$ denotes the concatenation of the feature map for the output of layers 0 to n. Inspired by Dense, DenseGCN applies a similar idea to GCN, i.e., let the output of the $k$-th layer contains transformations from all previous GCN layers to exploit the information from different GCN layers:
\begin{displaymath}
\mathbf{H}_k=\mathbf{AGG}_{dense}(\mathbf{H},\mathbf{H}_{1}, \ldots,\mathbf{H}_{k-1}). 
\end{displaymath}

\textbf{JK.} At the last layer, JK sifts from all previous representations $[\mathbf{X}_{1}, \ldots,\mathbf{X}_{N}]$ and combines them:
\begin{displaymath}
\mathbf{X}_{output}=\mathbf{AGG}(\mathbf{X}_{1}, \ldots,\mathbf{X}_{N}).    
\end{displaymath}

The $\mathcal{AGG}$ operation includes concatenation, Maxpooling and LSTM-attention. When it is introduced to GNN, i.e., JKNet, can be formulated as: 
\begin{displaymath}
\mathbf{H}_{output}=\mathbf{AGG}_{jk}(\mathbf{H}_1, \ldots,\mathbf{H}_{k-1}).  
\end{displaymath}

\section{SMV for Node Groups of Different Degrees}
\label{app_c}
\begin{figure}[H]
        \centering
        \includegraphics[width=\linewidth]{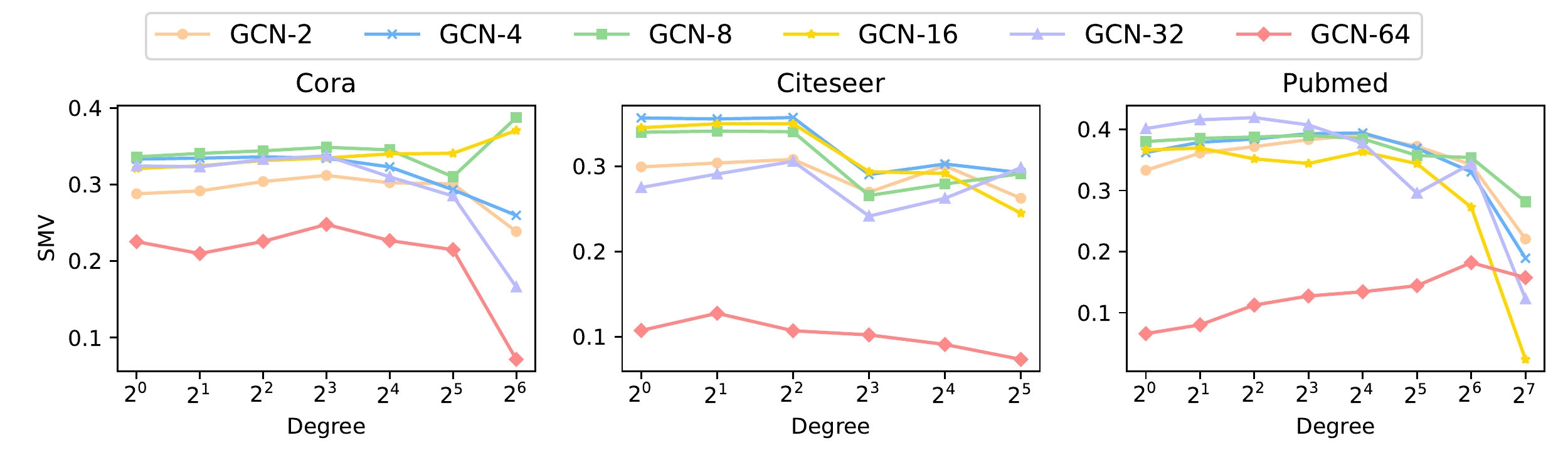}
        \caption*{Result of GCN.}
        \includegraphics[width=\linewidth]{Pict/gat_smv.pdf}
    \caption*{Result of GAT.}
\end{figure}

\section{Derivation of the closed-form formulas in the table}
\label{app_d}
\textit{ResGCN}:
We can write the recursive formula for ResGCN in the following form:
\begin{equation}
    \mathbf{H}_k = (\mathbf{I}+\mathbf{N})\mathbf{H}_{k-1}.
\end{equation}
In turn, the following form can be obtained by recursion:
\begin{equation}
    \mathbf{H}_k = (\mathbf{I}+\mathbf{N})^k\mathbf{H}.
\end{equation}
Using the binomial theorem, we can obtain the closed-form formula for ResGCN as follows:
\begin{equation}
\mathbf{H}_k = \sum^k_{j=0} \mathbf{C}_k^j\mathbf{N}^j\mathbf{H}.
\end{equation}
\textit{APPNP}:
    According to the recurrence formula of APPNP:
    \begin{align}
            \mathbf{H}_k = \alpha\mathbf{H} + (1-\alpha)\mathbf{N}\mathbf{H}_{k-1}.
    \end{align}
    To obtain the closed-form formula, we can add a term $\mathbf{T}$ to both sides of the equation:
    \begin{align}
        \mathbf{H}_k + \mathbf{T} = \left(1-\alpha\right)\mathbf{N}\mathbf{H}_{k-1} + \alpha\mathbf{H} + \mathbf{T}.
    \end{align}
    We aim to transform the equation into the following form:
    \begin{equation}
        \mathbf{H}_k + \mathbf{T} = 
        \left(
        1-\alpha
        \right)\mathbf{N} \left( \mathbf{H}_{k-1} + \mathbf{T}\right).
        \label{appnp-ideal}
    \end{equation}
    Then we need to make sure that there exists a very $\mathbf{T}$ that satisfies the following equation:
    \begin{equation}
        \left(
        1-\alpha
        \right)
        \mathbf{NT} = \alpha\mathbf{H}+ \mathbf{T},
    \end{equation}
    which can be transformed into the following form:
    \begin{equation}
        \left(
        \left(1-\alpha\right)\mathbf{N} - \mathbf{I}
        \right) \mathbf{T} = \alpha \mathbf{H}.
    \end{equation}
    We can proof the following lemma:
    \begin{lemma}
        Given that $\alpha \in \left(0,1\right)$,
        $\left(1-\alpha\right)\mathbf{N} - \mathbf{I}$ is invertible.
    \end{lemma}
    \begin{proof}
Proving that $ \left(1-\alpha\right)\mathbf{N} - \mathbf{I}$ is invertible is equivalent to demonstrating that it does not possess an eigenvalue of 0.
Consider the Rayleigh quotient of $\left(1-\alpha\right)\mathbf{N} - \mathbf{I}$:
    \begin{equation}
        \frac {
        \mathbf{X}^T \left(
        \left(1-\alpha\right)\mathbf{N} - \mathbf{I}
        \right)\mathbf{X}} {\mathbf{X}^T \mathbf{X}}
        = 
        \left(
        1-\alpha
        \right)\frac {
        \mathbf{X}^T 
        \left(
        \tilde{\mathbf{D}}^{-\frac{1}{2}}
        \tilde{\mathbf{A}}
        \tilde{\mathbf{D}}^{-\frac{1}{2}}
        \right)
        \mathbf{X}} {\mathbf{X}^T \mathbf{X}}-1.
        \label{appnplemma1}
    \end{equation}
    From spectral graph theory, we can know the following equation holds:
    \begin{equation}
        \mathbf{X}^T 
        \left(
        \tilde{\mathbf{D}}^{-\frac{1}{2}}
        \mathbf{L}
        \tilde{\mathbf{D}}^{-\frac{1}{2}}
        \right)
        \mathbf{X}
        =
        \sum_{(v_i,v_j) \in \mathcal{E}} \left(
        \frac
        {
        \mathbf{X}_i
        }
        {
        \sqrt{d_i+1}
        }
        -
        \frac
        {
        \mathbf{X}_j
        }
        {
        \sqrt{d_j+1}
        }
        \right)^2 > 0.
    \end{equation}
    We can decompose $\mathbf{L}$ into $\tilde{\mathbf{D}}-\tilde{\mathbf{A}}$, then we have:
    \begin{equation}
        \frac {
        \mathbf{X}^T 
        \left(
        \tilde{\mathbf{D}}^{-\frac{1}{2}}
        \tilde{\mathbf{D}}
        \tilde{\mathbf{D}}^{-\frac{1}{2}}
        \right)
        \mathbf{X}} {\mathbf{X}^T \mathbf{X}}
        -
        \frac {
        \mathbf{X}^T 
        \left(
        \tilde{\mathbf{D}}^{-\frac{1}{2}}
        \tilde{\mathbf{A}}
        \tilde{\mathbf{D}}^{-\frac{1}{2}}
        \right)
        \mathbf{X}} {\mathbf{X}^T \mathbf{X}} > 0,
    \end{equation}
    which is equivalent to:
    \begin{equation}
                \frac {
        \mathbf{X}^T 
        \left(
        \tilde{\mathbf{D}}^{-\frac{1}{2}}
        \tilde{\mathbf{A}}
        \tilde{\mathbf{D}}^{-\frac{1}{2}}
        \right)
        \mathbf{X}} {\mathbf{X}^T \mathbf{X}}
        <
        \frac {
        \mathbf{X}^T 
        \mathbf{I}
        \mathbf{X}} {\mathbf{X}^T \mathbf{X}}
        =
        1.
        \label{appnplemma2}
    \end{equation}
    Combining Eq.~\ref{appnplemma1} and Ineq.~\ref{appnplemma2}, we can obtain:
    \begin{equation}
        \frac {
        \mathbf{X}^T \left(
        \left(1-\alpha\right)\mathbf{N} - \mathbf{I}
        \right)\mathbf{X}} {\mathbf{X}^T \mathbf{X}}
        <\left(1-\alpha\right) - 1 = -\alpha < 0.
    \end{equation}
    Hence, 0 can't be the eigenvalue of $ \left(1-\alpha\right)\mathbf{N} - \mathbf{I}$ . Therefore, $ \left(1-\alpha\right)\mathbf{N} - \mathbf{I}$ is invertible.  \end{proof}
    Since \textbf{Lemma 1} holds, we can derive the concrete form of $\mathbf{T}$: 
    \begin{equation}
        \mathbf{T} = \alpha\left( 
        \left(1-\alpha\right)\mathbf{N} - \mathbf{I}
        \right)^{-1}\mathbf{H}.
    \end{equation}

    Thus we can keep recurring from Eq.~\ref{appnp-ideal} and obtain the following equation:
    \begin{equation}
                \mathbf{H}_k + \mathbf{T} = 
                \left(
                \left(
        1-\alpha
        \right)\mathbf{N}
        \right)^k
        \left(
        \mathbf{H} + \mathbf{T}
        \right),
        \label{ideal}
    \end{equation}
    which also can be written as:
    \begin{equation}
        \mathbf{H}_k =  
                \left(
                \left(
        1-\alpha
        \right)\mathbf{N}
        \right)^k\mathbf{H} +
                        \left(
                \left(
        1-\alpha
        \right)\mathbf{N}
        \right)^k\mathbf{T} -\mathbf{T}.
        \label{ageneral1}
    \end{equation}
    For the second and third terms in Eq.~\ref{ageneral1}, we write $\left(1-\alpha
        \right)\mathbf{N}$ as $\left(1-\alpha
        \right)\mathbf{N} - \mathbf{I} +\mathbf{I}$. We can use the binomial theorem to write $\left(
                \left(
        1- \alpha
        \right)\mathbf{N}
        \right)^k$ as $\sum\limits^k_{j=0} \left(
        \left(
        1-\alpha
        \right)\mathbf{N} - \mathbf{I}
        \right)^j,$
        then Eq.~\ref{ageneral1} can be written as :
        \begin{equation}
                    \mathbf{H}_k = 
                \left(
                \left(
        1-\alpha
        \right)\mathbf{N}
        \right)^k\mathbf{H}
        +
        \sum\limits^k_{j=1} 
        \left(
        \left(
        1-\alpha
        \right)
        \mathbf{N} - \mathbf{I}
        \right)^j \mathbf{T}.        
        \end{equation}
    Bring in the specific form of $\mathbf{T}$ and further derive the closed-form formula of APPNP:
    \begin{align}
        \mathbf{H}_k 
        &= 
                \left(
                \left(
        1-\alpha
        \right)\mathbf{N}
        \right)^k\mathbf{H}
        +
        \alpha\sum\limits^{k-1}_{j=0} 
        \left(
        \left(
        1-\alpha
        \right)
        \mathbf{N} - \mathbf{I}
        \right)^j \mathbf{H} \\
        & = 
        \left(
        1-\alpha
        \right)^k
        \mathbf{N}
        ^k\mathbf{H}
        +
        \alpha
        \sum\limits_{j=0}^{k-1}\sum\limits_{i=0}^{j}
        \left( -1\right)^{j-i}
        \left(1-\alpha\right)^i\mathbf{N}^i\mathbf{H}.
    \end{align}
\section{Derivation of the closed-form formula of PSNR-GCN}
\label{app_e}
For each diagonal element $\Lambda_{k,ii}$ of $\Lambda_k$, it is trivial to obtain: 
\begin{displaymath}
    0 < \Lambda_{k,ii} < 1 .
\end{displaymath}
To derive the closed-form formula of PSNR-GCN, we need proof the following lemma first.
\begin{lemma}
Set all the  diagonal elements of $\Lambda$ to satisfy
$0 < \Lambda_{ii} < 1$, then
$\left( 
\Lambda
\mathbf{N} + \mathbf{I}\right)$ is invertible.
\end{lemma} 
\begin{proof}
    Proving that $ \Lambda\mathbf{N} + \mathbf{I}$ is invertible is equivalent to demonstrating that its determinant is not equal to 0.
    Because all the  diagonal elements of $\Lambda$ satisfy
    $0 < \Lambda_{ii} < 1$, then $\Lambda$ is invertible. And due to
    \begin{equation}
        |\Lambda\mathbf{N} + \mathbf{I}| = |\Lambda||\mathbf{N} + \Lambda^{-1}|.
        \label{ap1-1}
    \end{equation}
Therefore, proving that its determinant is not equal to 0 is equivalent to demonstrating that $|\mathbf{N} + \Lambda^{-1}|$ is not equal to 0, and further equivalent to demonstrating that $\mathbf{N} + \Lambda^{-1}$ does not have an eigenvalue of 0.

Consider the Rayleigh quotient of $\mathbf{N} + \Lambda^{-1}$:
    \begin{equation}
        \mathbf{R}_1 = \frac {
        \mathbf{X}^T \left(
        \mathbf{N} + \Lambda^{-1} 
        \right)\mathbf{X}} {\mathbf{X}^T \mathbf{X}}
        \label{ap1-2}.
    \end{equation}
    Split Eq.~\ref{ap1-2}, and we can derive:
    \begin{equation}
        \mathbf{R}_1 = \frac {
        \mathbf{X}^T 
        \mathbf{N}\mathbf{X}
        } {\mathbf{X}^T \mathbf{X}}
        + \frac {
        \mathbf{X}^T 
        \Lambda^{-1} \mathbf{X}
        } {\mathbf{X}^T \mathbf{X}}.
        \label{ap1-3}
    \end{equation}
The second term of Eq.~\ref{ap1-3} can be easily written as follows:
\begin{displaymath}
    \frac {
        \mathbf{X}^T 
        \Lambda^{-1} \mathbf{X}
        } {\mathbf{X}^T \mathbf{X}} = \frac{\sum_{i=1}^{N} {\Lambda_{ii}}^{-1}x_i^2}
        {\sum_{i=1}^N x_i^2}
        \label{ap1-4}.
\end{displaymath}
Since $0<\Lambda_{ii}<1$, therefore  ${\Lambda_{ii}}^{-1} >1$, then
\begin{equation}
 \frac {
        \mathbf{X}^T 
        \Lambda^{-1} \mathbf{X}
        } {\mathbf{X}^T \mathbf{X}} > 1.
    \label{ap1-5}
\end{equation}
For the first item, we write its specific form as follows:
\begin{equation}
    \frac {
        \mathbf{X}^T 
        \mathbf{N}\mathbf{X}
        } {\mathbf{X}^T \mathbf{X}} = \frac
        {
        \mathbf{X}^T 
                \left(\tilde{\mathbf{D}}^{-\frac{1}{2}} \tilde{\mathbf{A}} \tilde{\mathbf{D}}^{-\frac{1}{2}}  \right)
         \mathbf{X}
        }
        {
        \mathbf{X}^T  \mathbf{X}
        }.
        \label{ap1-6}
\end{equation}
From spectral graph theory, we know that the following formula holds:
\begin{equation}
        \mathbf{X}^T 
        \left(\tilde{\mathbf{D}}^{-\frac{1}{2}} 
        \left(
        \mathbf{A} + \mathbf{D}
        \right)
        \tilde{\mathbf{D}}^{-\frac{1}{2}}  \right)
        \mathbf{X}
        =
        \sum_{(v_i,v_j) \in \mathcal{E}} \left(
        \frac
        {
        \mathbf{X}_i
        }
        {
        \sqrt{d_i+1}
        }
        +
        \frac
        {
        \mathbf{X}_j
        }
        {
        \sqrt{d_j+1}
        }
        \right)^2 > 0.
        \label{ap1-7}
\end{equation}
Further mathematically transforming this formula, we can get the following form:
\begin{flalign*}
     &  \frac {
        \mathbf{X}^T 
        \left(\tilde{\mathbf{D}}^{-\frac{1}{2}} 
        \left(
        \mathbf{A} + \mathbf{D}
        \right)
        \tilde{\mathbf{D}}^{-\frac{1}{2}}  \right)
        \mathbf{X}
        } {\mathbf{X}^T \mathbf{X}} &
\end{flalign*}
    \begin{align}
        &=
                \frac {
        \mathbf{X}^T 
        \left(\tilde{\mathbf{D}}^{-\frac{1}{2}} 
        \left(
        \tilde{\mathbf{A}} + \tilde{\mathbf{D}} - 2\mathbf{I}
        \right)
        \tilde{\mathbf{D}}^{-\frac{1}{2}}  \right)
        \mathbf{X}
        } {\mathbf{X}^T \mathbf{X}} \\
        &=
        \frac
        {
        \mathbf{X}^T 
                \left(\tilde{\mathbf{D}}^{-\frac{1}{2}} \tilde{\mathbf{A}} \tilde{\mathbf{D}}^{-\frac{1}{2}}  \right)
         \mathbf{X}
        }
        {
        \mathbf{X}^T  \mathbf{X}
        }
        +\frac
        {
        \mathbf{X}^T 
                \left(\tilde{\mathbf{D}}^{-\frac{1}{2}} \tilde{\mathbf{D}} \tilde{\mathbf{D}}^{-\frac{1}{2}}  \right)
         \mathbf{X}
        }
        {
        \mathbf{X}^T  \mathbf{X}
        }
        -\frac {
        2\mathbf{X}^T 
        \tilde{\mathbf{D}}^{-1}
        \mathbf{X}
        } {\mathbf{X}^T \mathbf{X}}       \\
        &= \frac
        {
        \mathbf{X}^T 
                \left(\tilde{\mathbf{D}}^{-\frac{1}{2}} \tilde{\mathbf{A}} \tilde{\mathbf{D}}^{-\frac{1}{2}}  \right)
         \mathbf{X}
        }
        {
        \mathbf{X}^T  \mathbf{X}
        }
        +1
        -\frac {
        2\mathbf{X}^T 
        \tilde{\mathbf{D}}^{-1}
        \mathbf{X}
        } {\mathbf{X}^T \mathbf{X}}  
         > 0.
    \end{align}
Further, we get the following result:
\begin{equation}
\frac
        {
        \mathbf{X}^T 
                \left(\tilde{\mathbf{D}}^{-\frac{1}{2}} \tilde{\mathbf{A}} \tilde{\mathbf{D}}^{-\frac{1}{2}}  \right)
         \mathbf{X}
        }
        {
        \mathbf{X}^T  \mathbf{X}
        }
        > \frac {
        2\mathbf{X}^T 
        \tilde{\mathbf{D}}^{-1}
        \mathbf{X}
        } {\mathbf{X}^T \mathbf{X}} -1.
        \label{ap1-8}
\end{equation}
It is trivial to obtain:
\begin{equation}
    \frac {
        2\mathbf{X}^T 
        \tilde{\mathbf{D}}^{-1}
        \mathbf{X}
        } {\mathbf{X}^T \mathbf{X}} = 
        \frac{2 \sum_{i=1}^{N} 
        \left( 
        \mathbf{d}_i+1
        \right)^{-1}
        \mathbf{x}_i^2}
        {\sum_{i=1}^N x_i^2} > 0 .      
        \label{ap1-9}
\end{equation}
Combining Eq.~\ref{ap1-3}, Ineq.~\ref{ap1-5}, Ineq.~\ref{ap1-8} and Ineq.~\ref{ap1-9}, we can get the following inequality:
\begin{equation}
    \frac
        {
        \mathbf{X}^T 
                \left(\tilde{\mathbf{D}}^{-\frac{1}{2}} \tilde{\mathbf{A}} \tilde{\mathbf{D}}^{-\frac{1}{2}} 
                + \Lambda^{-1}
                \right)
         \mathbf{X}
        }
        {
        \mathbf{X}^T  \mathbf{X}
        } > 0 .
\end{equation}
It can be obtained that the eigenvalue of $ \tilde{\mathbf{D}}^{-\frac{1}{2}} \tilde{\mathbf{A}} \tilde{\mathbf{D}}^{-\frac{1}{2}} + \Lambda^{-1} $ is greater than 0, so 0 is not an eigenvalue of it. Further, $ \Lambda\mathbf{N} + \mathbf{I}$ is invertible.
\end{proof}

Now, we derive the closed form of the formula.
Given the following recursive formula:
\begin{equation}
\mathbf{H}_{k}=
\mathbf{H}_{1} + 
\Lambda_{k-1}
\left(
\mathbf{H}_1 -
\tilde{\mathbf{D}}^{-1 / 2} \tilde{\mathbf{A}} \tilde{\mathbf{D}}^{-1 / 2}\mathbf{H}_{k-1} 
\right),
\label{ae1}
\end{equation}
where 
$\mathbf{H}_{1} = \tilde{\mathbf{D}}^{-1 / 2} \tilde{\mathbf{A}} \tilde{\mathbf{D}}^{-1 / 2}\mathbf{H} $ , $\mathbf{\Lambda}_{k} = diag\{\mathbf{\lambda}^{(1)}_{k},...\mathbf{\lambda}^{(n)}_{k}\}$,
$\mathbf{\lambda}_{k}^{(i)} \sim Sigmoid(\mathbf{\mathcal{N}}(\mathbf{\alpha}^{(i)}_{k},{\mathbf{\beta}^{(i)}_{k}}^{2}))$. After mathematical transformation, Eq.~\ref{ae1} can be written as:
\begin{equation}
\mathbf{H}_{k}=
\left(
\mathbf{I} + \Lambda_{k-1}
\right)
\mathbf{H}_{1} -  
\Lambda_{k-1}
\tilde{\mathbf{D}}^{-1 / 2} \tilde{\mathbf{A}} \tilde{\mathbf{D}}^{-1 / 2}\mathbf{H}_{k-1}.
\label{ae2}
\end{equation}
Set $\mathbf{N} = \tilde{\mathbf{D}}^{-1 / 2} \tilde{\mathbf{A}} \tilde{\mathbf{D}}^{-1 / 2}$, then Eq.~\ref{ae2} can be abbreviated as:
\begin{equation}
\mathbf{H}_{k}=
\left(
\mathbf{I} + \Lambda_{k-1}
\right)
\mathbf{H}_{1} -  
\Lambda_{k-1}
\mathbf{N}\mathbf{H}_{k-1}.
\label{ae3}
\end{equation}
We try to modify Eq.~\ref{ae3} to the form that is more suitable for obtaining the closed form of the formula:
\begin{equation}
\mathbf{H}_{k} + \mathbf{M}_{k-1}= -
\Lambda_{k-1}
\mathbf{N}
\left(
\mathbf{H}_{k-1} + \mathbf{M}_{k-1}
\right)
\label{ae4}.
\end{equation} 
To verify whether there exists such $\mathbf{M}$ that satisfies the equation, we need to solve the following equation:
\begin{equation}
 -  
\Lambda_{k-1}
\mathbf{N}\mathbf{M}_{k-1} =
\left(
\mathbf{I} + \Lambda_{k-1}
\right)
\mathbf{H}_{1} + \mathbf{M}_{k-1},
\label{ae5}
\end{equation}
which is equivalent to the following form:
\begin{equation}
 - \left( 
\Lambda_{k-1}
\mathbf{N} + \mathbf{I}\right)
\mathbf{M}_{k-1} =
\left(
\mathbf{I} + \Lambda_{k-1}
\right)
\mathbf{H}_{1}.
\label{ae6}
\end{equation}
Based on the definition, all the diagonal elements of $\Lambda_{k}$ satisfy
$0 < \lambda^{(i)}_{k} < 1$, so
according to \textbf{Lemma 2}, $\left( 
\Lambda_{k-1}
\mathbf{N} + \mathbf{I}\right)$ is invertible. Then $\mathbf{M}_{k-1} = -\left( 
\Lambda_{k-1}
\mathbf{N} + \mathbf{I}\right)^{-1}\left(
\mathbf{I} + \Lambda_{k-1}
\right)\mathbf{H}_{1},$
which means such $\mathbf{M}_{k-1}$ that we require exists.

First, we perform the following mathematical transformation on Eq.~\ref{ae4}:
\begin{equation} 
\mathbf{H}_k +\mathbf{M}_{k-1} = -\Lambda_{k-1}
\mathbf{N}
\left( 
\mathbf{H}_{k-1} + \mathbf{M}_{k-2} + \mathbf{M}_{k-1} - \mathbf{M}_{k-2}
\right),
\label{ae7}
\end{equation}
which can be split into the following form:
\begin{equation}
    \mathbf{H}_k +\mathbf{M}_{k-1} = -\Lambda_{k-1}
\mathbf{N}
\left( 
\mathbf{H}_{k-1} + \mathbf{M}_{k-2}\right) + \left(-\Lambda_{k-1}
\mathbf{N}\right)\left(\mathbf{M}_{k-1} - \mathbf{M}_{k-2}
\right).
\label{ae8}
\end{equation}
Let $\tilde{\mathbf{N}}_{k-1}$ denotes $-\Lambda_{k-1}
\mathbf{N}$, so the formula can be simply written as:
\begin{equation}
    \mathbf{H}_k +\mathbf{M}_{k-1} = \tilde{\mathbf{N}}_{k-1}
\left( 
\mathbf{H}_{k-1} + \mathbf{M}_{k-2}\right) + \tilde{\mathbf{N}}_{k-1}\left(\mathbf{M}_{k-1} - \mathbf{M}_{k-2}
\right).
\label{ae9}
\end{equation}
We first use Eq.~\ref{ae4} to recurse once, then derive the following formula:
\begin{equation}
    \mathbf{H}_k +\mathbf{M}_{k-1} = \tilde{\mathbf{N}}_{k-1}\tilde{\mathbf{N}}_{k-2}
\left( 
\mathbf{H}_{k-2} + \mathbf{M}_{k-3}\right) + \tilde{\mathbf{N}}_{k-1}\left(\mathbf{M}_{k-1} - \mathbf{M}_{k-2}
\right).
\label{ae10}
\end{equation}
By analogy, continuing to split and iterate, we can get the closed form of the output of the $k$-th layer:
\begin{equation}
     \mathbf{H}_k  = 
\sum_{i=2}^{k-1} \prod_{j=i}^{k-1}\tilde{\mathbf{N}}_j
\left(
\mathbf{M}_i - \mathbf{M}_{i-1}
\right)
+ \prod_{i=1}^{k-1}\tilde{\mathbf{N}}_i\left( \mathbf{H}_1 + \mathbf{M}_1\right)
-\mathbf{M}_{k-1}.
\label{ae11}
\end{equation}


\section{Comparing with other subgraph-based methods}
\label{app_f}
It is worth noting that a lot of previous work mentions "subgraph". However, our approach is fundamentally different from these approaches.

\textbf{Relation with other subgraph-based methods.} While there are existing works \cite{DBLP:conf/iclr/BevilacquaFLSCB22, DBLP:conf/nips/QianRGN022} related to subgraphs, they primarily focus on graph classification tasks, aiming to learn representations of entire graphs. Given the limited capacity of GNNs to effectively represent the entire graphs, these subgraph-based approaches employ various strategies to leverage information from multiple subgraph structures within the overall graph to improve the representation.

In contrast, our focus is on node-level tasks, specifically, enhancing the representations of individual nodes. In our context, we naturally refer to neighborhood subgraphs as different orders of ego-networks centered on a node (which can be regarded as subgraphs of the entire graph). Our work uncovers the relationship between these subgraphs and over-smoothing, as well as how to utilize them to enhance node representations. This distinction fundamentally sets our approach apart from other subgraph methods.


\section{Experiment Setup}
\label{app_g}
\subsection{Details of Datasets}
The dataset statistics are shown in Table~\ref{tab:datasettab}, and details on dataset splits are summarized as follows:

\textbf{Experiment } \ref{exp:oversmoothing} \& \ref{exp:full}. For Cora, Citeseer, Coauthor-CS, Amazon-Photo, we randomly select 20 nodes per class for training set, 500 nodes for validation and 1000 nodes for testing. For Chameleon and Squirrel, we randomly divide each
class’s nodes into 60\%, 20\%, and 20\% as the train, validation, and test
sets, respectively. 
 
\textbf{Experiment } \ref{exp:missing}.
We follow the widely used semi-supervised setting in \cite{DBLP:conf/iclr/KipfW17}.

\textbf{Experiment } \ref{ls}.
For larger datasets, We randomly divide each class’s nodes into 20\%, 20\%, and 60\% as the train, validation, and test sets, respectively.

\begin{table}[H]
  \caption{Dataset statistics of real-world datasets.}
  \label{tab:datasettab}
  \centering
    \resizebox{\linewidth}{!}{
  \begin{tabular}{ccccccccccccc}
    \toprule 
        & \textbf{Cora} & \textbf{Citeseer} &\textbf{Pubmed}& \textbf{Amazon-Photo} &\textbf{Ogbn-arxiv} &\textbf{Chameleon} &\textbf{Squirrel}& \textbf{Coauthor-CS} & \textbf{Coauthor-Phy} & \textbf{Flickr}\\ 
      \midrule 
      \text{\#Nodes} &2708 &3327&19717&7650&169343&2277&5201 &18333 &34493 &89250 \\
    \text{\#Edges}&5429&4732&119081&126842&1166243&36101&217073 &81894 & 247962 &899756 \\

        \text{\#Features} &1433&3703&500&745&128&2325&2089 & 6805 &8415 &500 \\
                \text{\#Classes} &7&6&3&8&39&5&5 & 5 & 10  & 7 \\
      \bottomrule 
  \end{tabular}
  }
\end{table}

\subsection{Parameter Settings} 
We summarized the hyperparameters used in different experiments in Table \ref{tab:hyper}.
\begin{table}[ht]
\caption{Hyperparameters for experiments}
\label{tab:hyper}
\centering
\resizebox{\textwidth}{!}{
\begin{tabular}{ccccccccc}
\hline
\hline
\textbf{Experiment} & \textbf{Backbone} & \textbf{Learning Rate} & \textbf{Dropout} & \textbf{Weight Decay} & \textbf{Hidden State} & \textbf{Attention Head} & \textbf{Max Epoch} & \textbf{Early Stopping Epoch} \\ 
\midrule
\multirow{2}{*}{Experiment \ref{exp:oversmoothing} \& \ref{exp:full}} 
& GCN & \{0.01, 0.001\} & 0.5 & 0.0005 & 128 & - & 500 & 100\\
& GAT & \{0.01, 0.001\} & 0.5 & 0.0005 & 64 & 3 & 500 & 100\\
\midrule
\multirow{2}{*}{Experiment \ref{exp:missing}} 
& GCN & \{0.01, 0.001\} & 0.5 & 0.0005 & 128 & -  & 1000 & 1000\\
& GAT & \{0.01, 0.001\} & 0.5 & 0.0005 & 32 & 1  & 1000 & 1000\\
\midrule
\multirow{1}{*}{Experiment \ref{ls}} 
& GCN & \{0.01, 0.001\} & 0.5 & 0.0005 & 128 & - & 500 & 100 \\
\hline
\hline
\end{tabular}}
\label{tab:hyperparameters}
\end{table}

\section*{H. The Analysis of PSNR Residual Coefficients}
\label{Appendix H}
We conducted an empirical study using an 8-layer PSNR-GCN trained on the Cora dataset to obtain the best-performing model. We saved the mean and standard deviation of the learned residual coefficient distribution for each layer. Nodes were evenly divided into four groups based on their degree, with each group containing a similar number of nodes, and the average mean and standard deviation for each group across different layers are reported in Table \ref{tab:summary}. The following observations can be obtained from the table: the mean residual coefficient increases with the number of layers, suggesting that PSNR effectively retains high-order subgraph information. In certain layers, the variance of the residual coefficients rises, indicating that added randomness helps mitigate information loss in higher-order subgraphs. In the shallow layers, the mean values show no significant differences across node degrees; however, in deeper layers, nodes with higher degrees tend to exhibit lower mean values, indicating that these nodes retain more initial information due to significant subgraph overlap. All of these observations align with our expectations, illustrating how residual coefficients adapt based on node degree and layer depth.

\begin{table}[ht]
    \centering
    \caption{Summary of means and standard deviations for different layers}
    \resizebox{\textwidth}{!}{
    \begin{tabular}{|c|c|c|c|c|c|c|c|c|}
        \hline
        \textbf{Layer} & \textbf{Mean Group 1} & \textbf{Mean Group 2} & \textbf{Mean Group 3} & \textbf{Mean Group 4} & \textbf{Std Group 1} & \textbf{Std Group 2} & \textbf{Std Group 3} & \textbf{Std Group 4} \\
        \hline
        Layer 0 & 0.0043 & 0.0049 & 0.0063 & 0.0099 & 0.0002 & 0.0001 & 0.0002 & 0.0001  \\
        \hline
        Layer 1 & 0.0009 & 0.0017 & 0.0013 & 0.0016 & 0.3102 & 0.3120 & 0.3130& 0.2410\\
        \hline
        Layer 2 & 0.0000 & 0.0000 & 0.0003 & 0.0008 & 0.0480 & 0.0431 & 0.0446& 0.0170\\
        \hline
        Layer 3 & 0.5582 & 0.6017 & 0.5972 & 0.4979 & 0.0084 & 0.0118 & 0.0105& 0.0080\\
        \hline
        Layer 4 & 1.4898 & 1.5238 & 1.5212 & 1.3478 & 8.3354e-05 & 5.4541e-05 & 0.0001& 0.0002\\
        \hline
        Layer 5 & 2.4080 & 2.3974 & 2.3981 & 2.1414 & 0.0062 & 0.0081 & 0.0078& 0.0059\\
        \hline
        Layer 6 & 5.2347 & 5.3148 & 5.2070 & 3.8177 & 0.0046 & 0.0089 & 0.0075& 0.0069\\
        \hline
        Layer 7 & 11.0787 & 11.1253 & 10.9247 & 8.7238 & 0.1642& 0.2089& 0.2128& 0.0677\\
        \hline
    \end{tabular}}
    \label{tab:summary}
\end{table}

\section*{I. Comparison of Experimental Results of Different GraphEncoders.}
\label{appi}
In practice, in addition to SAGE, encoders such as GAT and GCN can also be utilized. To provide further insight, we have included results for different encoders on the classical semi-supervised node classification task. The results are summarized in the table blow.

\begin{table}[h]
\centering
\caption{Different GraphEncoder performance for SSNC task (layer 2)}
\resizebox{\textwidth}{!}{
\begin{tabular}{lcccccc}
\toprule
Graph Encoder & Cora          & Citeseer      & CS            & Photo         & Chameleon     & Squirrel     \\
\midrule
GCN           & \textbf{80.98$\pm$1.60} & 68.46$\pm$2.28 & 90.52$\pm$0.82 & \textbf{91.56$\pm$0.74} & \textbf{72.02$\pm$1.60} & 56.14$\pm$1.51 \\
GAT           & 80.89$\pm$1.63 & \textbf{68.77$\pm$1.89} & 90.61$\pm$0.89 & 91.18$\pm$0.92 & 71.97$\pm$1.28 & \textbf{56.24$\pm$1.11} \\
SAGE          & 80.59$\pm$1.57 & 68.06$\pm$2.12 & \textbf{91.23$\pm$1.00} & 91.44$\pm$0.82 & 71.51$\pm$1.90 & 54.95$\pm$1.73 \\
\bottomrule
\end{tabular}}
\label{graphencoder1}
\end{table}

\begin{table}[h]
\centering
\caption{Different GraphEncoder performance for SSNC task (layer 4)}
\resizebox{\textwidth}{!}{
\begin{tabular}{lcccccc}
\toprule
Graph Encoder & Cora          & Citeseer      & CS            & Photo         & Chameleon     & Squirrel     \\
\midrule
GCN           & 81.65$\pm$1.70 & \textbf{68.11$\pm$1.24} & 90.66$\pm$0.70 & 91.14$\pm$0.90 & \textbf{71.58$\pm$2.07} & 56.34$\pm$1.48 \\
GAT           & \textbf{82.21$\pm$1.41} & 67.96$\pm$1.20 & 90.57$\pm$0.89 & 91.17$\pm$0.81 & 71.29$\pm$1.75 & \textbf{56.50$\pm$1.45} \\
SAGE          & 81.01$\pm$1.63 & 66.03$\pm$1.93 & \textbf{90.70$\pm$1.49} & \textbf{91.20$\pm$1.03} & 70.74$\pm$2.24 & 54.13$\pm$1.41 \\
\bottomrule
\end{tabular}}
\label{graphencoder2}
\end{table}

The results indicate that each encoder has its strengths and performs differently across various datasets, demonstrating superior performance compared to the baseline and emphasizing the potential of PSNR.

\newpage
\section*{NeurIPS Paper Checklist}

The checklist is designed to encourage best practices for responsible machine learning research, addressing issues of reproducibility, transparency, research ethics, and societal impact. Do not remove the checklist: {\bf The papers not including the checklist will be desk rejected.} The checklist should follow the references and precede the (optional) supplemental material.  The checklist does NOT count towards the page
limit. 

Please read the checklist guidelines carefully for information on how to answer these questions. For each question in the checklist:
\begin{itemize}
    \item You should answer \answerYes{}, \answerNo{}, or \answerNA{}.
    \item \answerNA{} means either that the question is Not Applicable for that particular paper or the relevant information is Not Available.
    \item Please provide a short (1–2 sentence) justification right after your answer (even for NA). 
\end{itemize}

{\bf The checklist answers are an integral part of your paper submission.} They are visible to the reviewers, area chairs, senior area chairs, and ethics reviewers. You will be asked to also include it (after eventual revisions) with the final version of your paper, and its final version will be published with the paper.

The reviewers of your paper will be asked to use the checklist as one of the factors in their evaluation. While "\answerYes{}" is generally preferable to "\answerNo{}", it is perfectly acceptable to answer "\answerNo{}" provided a proper justification is given (e.g., "error bars are not reported because it would be too computationally expensive" or "we were unable to find the license for the dataset we used"). In general, answering "\answerNo{}" or "\answerNA{}" is not grounds for rejection. While the questions are phrased in a binary way, we acknowledge that the true answer is often more nuanced, so please just use your best judgment and write a justification to elaborate. All supporting evidence can appear either in the main paper or the supplemental material, provided in appendix. If you answer \answerYes{} to a question, in the justification please point to the section(s) where related material for the question can be found.

IMPORTANT, please:
\begin{itemize}
    \item {\bf Delete this instruction block, but keep the section heading ``NeurIPS paper checklist"},
    \item  {\bf Keep the checklist subsection headings, questions/answers and guidelines below.}
    \item {\bf Do not modify the questions and only use the provided macros for your answers}.
\end{itemize}


\begin{enumerate}

\item {\bf Claims}
    \item[] Question: Do the main claims made in the abstract and introduction accurately reflect the paper's contributions and scope?
    \item[] Answer: \answerYes{} 
    \item[] Justification: Our paper’s contributions and scope are presented in the abstract and introduction accurately.

\item {\bf Limitations}
    \item[] Question: Does the paper discuss the limitations of the work performed by the authors?
    \item[] Answer: \answerYes{} 
    \item[] Justification:  Limitations of our work are discussed in Section \ref{sec: conclusion}.

\item {\bf Theory Assumptions and Proofs}
    \item[] Question: For each theoretical result, does the paper provide the full set of assumptions and a complete (and correct) proof?
    \item[] Answer: \answerYes{} 
    \item[] Justification: We provide the proofs of all theorems and the derivations of all formulas in Section \ref{sec:theory}, Appendix \ref{app_d} and Appendix \ref{app_e}, respectively.

    \item {\bf Experimental Result Reproducibility}
    \item[] Question: Does the paper fully disclose all the information needed to reproduce the main experimental results of the paper to the extent that it affects the main claims and/or conclusions of the paper (regardless of whether the code and data are provided or not)?
    \item[] Answer: \answerYes{} 
    \item[] Justification: We provide the details of experiments in Appendix \ref{app_g}, which can ensure the reproducibility of our experimental result.

\item {\bf Open access to data and code}
    \item[] Question: Does the paper provide open access to the data and code, with sufficient instructions to faithfully reproduce the main experimental results, as described in supplemental material?
    \item[] Answer: \answerYes{} 
    \item[] Justification: We only used public datasets and the anonymous link to the code is provided in the Abstract.

\item {\bf Experimental Setting/Details}
    \item[] Question: Does the paper specify all the training and test details (e.g., data splits, hyperparameters, how they were chosen, type of optimizer, etc.) necessary to understand the results?
    \item[] Answer: \answerYes{} 
    \item[] Justification: We provide the details of training and test in Appendix \ref{app_g}.

\item {\bf Experiment Statistical Significance}
    \item[] Question: Does the paper report error bars suitably and correctly defined or other appropriate information about the statistical significance of the experiments?
    \item[] Answer: \answerYes{} 
    \item[] Justification: The standard deviation of the experimental results is reported in most of the tables.

\item {\bf Experiments Compute Resources}
    \item[] Question: For each experiment, does the paper provide sufficient information on the computer resources (type of compute workers, memory, time of execution) needed to reproduce the experiments?
    \item[] Answer: \answerYes{} 
    \item[] Justification: We report the details about the compute resources in Section \ref{sec:resourse}
    
\item {\bf Code Of Ethics}
    \item[] Question: Does the research conducted in the paper conform, in every respect, with the NeurIPS Code of Ethics \url{https://neurips.cc/public/EthicsGuidelines}?
    \item[] Answer: \answerYes{} 
    \item[] Justification: Our research conducted in the paper conform, in every respect, with the NeurIPS Code of Ethics.

\item {\bf Broader Impacts}
    \item[] Question: Does the paper discuss both potential positive societal impacts and negative societal impacts of the work performed?
    \item[] Answer: \answerYes{} 
    \item[] Justification: 
    We have discussed the potential positive societal impacts in section \ref{sec: conclusion}. It has no negative societal impact.
    
\item {\bf Safeguards}
    \item[] Question: Does the paper describe safeguards that have been put in place for responsible release of data or models that have a high risk for misuse (e.g., pretrained language models, image generators, or scraped datasets)?
    \item[] Answer: \answerNA{} 
    \item[] Justification: \justificationTODO{}

\item {\bf Licenses for existing assets}
    \item[] Question: Are the creators or original owners of assets (e.g., code, data, models), used in the paper, properly credited and are the license and terms of use explicitly mentioned and properly respected?
    \item[] Answer: \answerYes{} 
    \item[] Justification: All the original owners of assets (e.g., code, data, models) used in this paper are properly credited, and the license and terms of use are explicitly mentioned and properly respected.


\item {\bf New Assets}
    \item[] Question: Are new assets introduced in the paper well documented and is the documentation provided alongside the assets?
    \item[] Answer: \answerYes{} 
    \item[] Justification: Yes, we provide the anonymous link to the code.

\item {\bf Crowdsourcing and Research with Human Subjects}
    \item[] Question: For crowdsourcing experiments and research with human subjects, does the paper include the full text of instructions given to participants and screenshots, if applicable, as well as details about compensation (if any)? 
    \item[] Answer: \answerNA{} 
    \item[] Justification: Our work does not involve 
 any crowdsourcing nor research with human subjects.

\item {\bf Institutional Review Board (IRB) Approvals or Equivalent for Research with Human Subjects}
    \item[] Question: Does the paper describe potential risks incurred by study participants, whether such risks were disclosed to the subjects, and whether Institutional Review Board (IRB) approvals (or an equivalent approval/review based on the requirements of your country or institution) were obtained?
    \item[] Answer: \answerNA{} 
    \item[] Justification:  The paper does not involve crowdsourcing or research with human subjects.

\end{enumerate}


\begin{thebibliography}{10}

\bibitem{DBLP:conf/icml/AzabouGTLSLVVD23}
Mehdi Azabou, Venkataramana Ganesh, Shantanu Thakoor, Chi{-}Heng Lin, Lakshmi Sathidevi, Ran Liu, Michal Valko, Petar Velickovic, and Eva~L. Dyer.
\newblock Half-hop: {A} graph upsampling approach for slowing down message passing.
\newblock In Andreas Krause, Emma Brunskill, Kyunghyun Cho, Barbara Engelhardt, Sivan Sabato, and Jonathan Scarlett, editors, {\em International Conference on Machine Learning, {ICML} 2023, 23-29 July 2023, Honolulu, Hawaii, {USA}}, volume 202 of {\em Proceedings of Machine Learning Research}, pages 1341--1360. {PMLR}, 2023.

\bibitem{DBLP:conf/iclr/BevilacquaFLSCB22}
Beatrice Bevilacqua, Fabrizio Frasca, Derek Lim, Balasubramaniam Srinivasan, Chen Cai, Gopinath Balamurugan, Michael~M. Bronstein, and Haggai Maron.
\newblock Equivariant subgraph aggregation networks.
\newblock In {\em The Tenth International Conference on Learning Representations, {ICLR} 2022, Virtual Event, April 25-29, 2022}. OpenReview.net, 2022.

\bibitem{DBLP:conf/icml/ChenWHDL20}
Ming Chen, Zhewei Wei, Zengfeng Huang, Bolin Ding, and Yaliang Li.
\newblock Simple and deep graph convolutional networks.
\newblock In {\em Proceedings of {ICML}}, volume 119, pages 1725--1735, 2020.

\bibitem{DBLP:conf/nips/DuvenaudMABHAA15}
David Duvenaud, Dougal Maclaurin, Jorge Aguilera{-}Iparraguirre, Rafael G{\'{o}}mez{-}Bombarelli, Timothy Hirzel, Al{\'{a}}n Aspuru{-}Guzik, and Ryan~P. Adams.
\newblock Convolutional networks on graphs for learning molecular fingerprints.
\newblock In {\em Proceedings of {NeurIPS}}, pages 2224--2232, 2015.

\bibitem{DBLP:conf/nips/EliasofHT21}
Moshe Eliasof, Eldad Haber, and Eran Treister.
\newblock {PDE-GCN:} novel architectures for graph neural networks motivated by partial differential equations.
\newblock In Marc'Aurelio Ranzato, Alina Beygelzimer, Yann~N. Dauphin, Percy Liang, and Jennifer~Wortman Vaughan, editors, {\em Advances in Neural Information Processing Systems 34: Annual Conference on Neural Information Processing Systems 2021, NeurIPS 2021, December 6-14, 2021, virtual}, pages 3836--3849, 2021.

\bibitem{DBLP:conf/www/Fan0LHZTY19}
Wenqi Fan, Yao Ma, Qing Li, Yuan He, Yihong~Eric Zhao, Jiliang Tang, and Dawei Yin.
\newblock Graph neural networks for social recommendation.
\newblock In {\em Proceedings of WWW}, pages 417--426, 2019.

\bibitem{DBLP:conf/aaai/FangX0XY023}
Taoran Fang, Zhiqing Xiao, Chunping Wang, Jiarong Xu, Xuan Yang, and Yang Yang.
\newblock Dropmessage: Unifying random dropping for graph neural networks.
\newblock In Brian Williams, Yiling Chen, and Jennifer Neville, editors, {\em Thirty-Seventh {AAAI} Conference on Artificial Intelligence, {AAAI} 2023, Thirty-Fifth Conference on Innovative Applications of Artificial Intelligence, {IAAI} 2023, Thirteenth Symposium on Educational Advances in Artificial Intelligence, {EAAI} 2023, Washington, DC, USA, February 7-14, 2023}, pages 4267--4275. {AAAI} Press, 2023.

\bibitem{DBLP:conf/nips/FengZDHLXYK020}
Wenzheng Feng, Jie Zhang, Yuxiao Dong, Yu~Han, Huanbo Luan, Qian Xu, Qiang Yang, Evgeny Kharlamov, and Jie Tang.
\newblock Graph random neural networks for semi-supervised learning on graphs.
\newblock In Hugo Larochelle, Marc'Aurelio Ranzato, Raia Hadsell, Maria{-}Florina Balcan, and Hsuan{-}Tien Lin, editors, {\em Advances in Neural Information Processing Systems 33: Annual Conference on Neural Information Processing Systems 2020, NeurIPS 2020, December 6-12, 2020, virtual}, 2020.

\bibitem{DBLP:conf/cvpr/HeZRS16}
Kaiming He, Xiangyu Zhang, Shaoqing Ren, and Jian Sun.
\newblock Deep residual learning for image recognition.
\newblock In {\em Proceedings of {CVPR}}, pages 770--778, 2016.

\bibitem{DBLP:journals/corr/abs-2005-00687}
Weihua Hu, Matthias Fey, Marinka Zitnik, Yuxiao Dong, Hongyu Ren, Bowen Liu, Michele Catasta, and Jure Leskovec.
\newblock Open graph benchmark: Datasets for machine learning on graphs.
\newblock {\em CoRR}, abs/2005.00687, 2020.

\bibitem{DBLP:conf/icml/IoffeS15}
Sergey Ioffe and Christian Szegedy.
\newblock Batch normalization: Accelerating deep network training by reducing internal covariate shift.
\newblock In {\em Proceedings of {ICML}}, volume~37, pages 448--456, 2015.

\bibitem{DBLP:conf/kdd/JinL0AT22}
Wei Jin, Xiaorui Liu, Yao Ma, Charu~C. Aggarwal, and Jiliang Tang.
\newblock Feature overcorrelation in deep graph neural networks: {A} new perspective.
\newblock In {\em Proceedings of {SIGKDD}}, pages 709--719, 2022.

\bibitem{DBLP:journals/corr/KingmaB14}
Diederik~P. Kingma and Jimmy Ba.
\newblock Adam: {A} method for stochastic optimization.
\newblock In {\em Proceedings of {ICLR}}, 2015.

\bibitem{DBLP:conf/iclr/KipfW17}
Thomas~N. Kipf and Max Welling.
\newblock Semi-supervised classification with graph convolutional networks.
\newblock In {\em Proceedings of {ICLR}}, 2017.

\bibitem{DBLP:conf/iclr/KlicperaBG19}
Johannes Klicpera, Aleksandar Bojchevski, and Stephan G{\"{u}}nnemann.
\newblock Predict then propagate: Graph neural networks meet personalized pagerank.
\newblock In {\em Proceedings of {ICLR}}, 2019.

\bibitem{DBLP:conf/iccv/Li0TG19}
Guohao Li, Matthias M{\"{u}}ller, Ali~K. Thabet, and Bernard Ghanem.
\newblock Deepgcns: Can gcns go as deep as cnns?
\newblock In {\em Proceedings of {ICCV}}, pages 9266--9275, 2019.

\bibitem{DBLP:conf/aaai/LiHW18}
Qimai Li, Zhichao Han, and Xiao{-}Ming Wu.
\newblock Deeper insights into graph convolutional networks for semi-supervised learning.
\newblock In {\em Proceedings of AAAI}, pages 3538--3545, 2018.

\bibitem{DBLP:conf/kdd/LiuHJLL23}
Hua Liu, Haoyu Han, Wei Jin, Xiaorui Liu, and Hui Liu.
\newblock Enhancing graph representations learning with decorrelated propagation.
\newblock In Ambuj~K. Singh, Yizhou Sun, Leman Akoglu, Dimitrios Gunopulos, Xifeng Yan, Ravi Kumar, Fatma Ozcan, and Jieping Ye, editors, {\em Proceedings of the 29th {ACM} {SIGKDD} Conference on Knowledge Discovery and Data Mining, {KDD} 2023, Long Beach, CA, USA, August 6-10, 2023}, pages 1466--1476. {ACM}, 2023.

\bibitem{DBLP:conf/kdd/LiuGJ20}
Meng Liu, Hongyang Gao, and Shuiwang Ji.
\newblock Towards deeper graph neural networks.
\newblock In Rajesh Gupta, Yan Liu, Jiliang Tang, and B.~Aditya Prakash, editors, {\em {KDD} '20: The 26th {ACM} {SIGKDD} Conference on Knowledge Discovery and Data Mining, Virtual Event, CA, USA, August 23-27, 2020}, pages 338--348. {ACM}, 2020.

\bibitem{mcauley2012image}
Julian McAuley and Jure Leskovec.
\newblock Image labeling on a network: using social-network metadata for image classification.
\newblock In {\em Computer Vision--ECCV 2012: 12th European Conference on Computer Vision, Florence, Italy, October 7-13, 2012, Proceedings, Part IV 12}, pages 828--841. Springer, 2012.

\bibitem{DBLP:conf/iclr/OonoS20}
Kenta Oono and Taiji Suzuki.
\newblock Graph neural networks exponentially lose expressive power for node classification.
\newblock In {\em Proceedings of {ICLR}}, 2020.

\bibitem{DBLP:conf/kdd/PerozziAS14}
Bryan Perozzi, Rami Al{-}Rfou, and Steven Skiena.
\newblock Deepwalk: Online learning of social representations.
\newblock In {\em Proceedings of {SIGKDD}}, pages 701--710, 2014.

\bibitem{DBLP:conf/nips/QianRGN022}
Chendi Qian, Gaurav Rattan, Floris Geerts, Mathias Niepert, and Christopher Morris.
\newblock Ordered subgraph aggregation networks.
\newblock In {\em NeurIPS}, 2022.

\bibitem{Rong_2019}
Y.~Rong, Wen-bing Huang, Tingyang Xu, and Junzhou Huang.
\newblock Dropedge: Towards deep graph convolutional networks on node classification.
\newblock {\em Proceedings of {ICLR}}, 2019.

\bibitem{DBLP:journals/compnet/RozemberczkiAS21}
Benedek Rozemberczki, Carl Allen, and Rik Sarkar.
\newblock Multi-scale attributed node embedding.
\newblock {\em Journal of Complex Networks}, 9(2), 2021.

\bibitem{DBLP:journals/corr/abs-2303-10993}
T.~Konstantin Rusch, Michael~M. Bronstein, and Siddhartha Mishra.
\newblock A survey on oversmoothing in graph neural networks.
\newblock {\em CoRR}, abs/2303.10993, 2023.

\bibitem{DBLP:journals/aim/SenNBGGE08}
Prithviraj Sen, Galileo Namata, Mustafa Bilgic, Lise Getoor, Brian Gallagher, and Tina Eliassi{-}Rad.
\newblock Collective classification in network data.
\newblock {\em {AI} Magazine}, 29(3):93--106, 2008.

\bibitem{shchur2018pitfalls}
Oleksandr Shchur, Maximilian Mumme, Aleksandar Bojchevski, and Stephan G{\"u}nnemann.
\newblock Pitfalls of graph neural network evaluation.
\newblock {\em arXiv preprint arXiv:1811.05868}, 2018.

\bibitem{vaswani2017attention}
Ashish Vaswani, Noam Shazeer, Niki Parmar, Jakob Uszkoreit, Llion Jones, Aidan~N Gomez, {\L}ukasz Kaiser, and Illia Polosukhin.
\newblock Attention is all you need.
\newblock {\em Advances in neural information processing systems}, 30, 2017.

\bibitem{Velickovic_2017}
Petar Velickovic, Guillem Cucurull, Arantxa Casanova, Adriana Romero, P.~Lio’, and Yoshua Bengio.
\newblock Graph attention networks.
\newblock {\em Proceedings of {ICLR}}, 2017.

\bibitem{DBLP:conf/nips/WuAWJ23}
Xinyi Wu, Amir Ajorlou, Zihui Wu, and Ali Jadbabaie.
\newblock Demystifying oversmoothing in attention-based graph neural networks.
\newblock In Alice Oh, Tristan Naumann, Amir Globerson, Kate Saenko, Moritz Hardt, and Sergey Levine, editors, {\em Advances in Neural Information Processing Systems 36: Annual Conference on Neural Information Processing Systems 2023, NeurIPS 2023, New Orleans, LA, USA, December 10 - 16, 2023}, 2023.

\bibitem{DBLP:conf/icml/XuLTSKJ18}
Keyulu Xu, Chengtao Li, Yonglong Tian, Tomohiro Sonobe, Ken{-}ichi Kawarabayashi, and Stefanie Jegelka.
\newblock Representation learning on graphs with jumping knowledge networks.
\newblock In {\em Proceedings of {ICML}}, pages 5449--5458, 2018.

\bibitem{Zhao_2019}
Lingxiao Zhao and L.~Akoglu.
\newblock Pairnorm: Tackling oversmoothing in gnns.
\newblock {\em Proceedings of {ICLR}}, 2019.

\bibitem{Zhou_2020}
Kaixiong Zhou, Xiao Huang, Yuening Li, D.~Zha, Rui Chen, and Xia Hu.
\newblock Towards deeper graph neural networks with differentiable group normalization.
\newblock {\em Proceedings of {NeurIPS}}, 2020.

\end{thebibliography}
\end{document}